%% file: manuscript.tex
\documentclass[letter,11pt]{article}
\usepackage[margin=1in]{geometry}
\usepackage[american]{babel}

\usepackage[numbers,sort&compress]{natbib}
\usepackage{xcolor}
\usepackage{hyperref}
    
\usepackage{mathtools} % amsmath with fixes and additions
\usepackage{booktabs} % commands to create good-looking tables
\usepackage{tikz} % nice language for creating drawings and diagrams

\usepackage{amsthm}
\usepackage{bbm}

\usepackage{algorithm}

\usepackage{amssymb}    % for \checkmark
\usepackage{algpseudocode}
\usepackage{amsmath}
\usepackage{booktabs}
\usepackage{multirow}

\newtheorem{theorem}{Theorem}

\newtheorem{proposition}{Proposition}
\newtheorem{definition}{Definition}
\newtheorem{assumption}{Assumption}
\newtheorem{corollary}{Corollary}

\usepackage{titletoc}

\DeclareMathOperator*{\argmax}{arg\,max}

\title{\vspace{-10mm}\textbf{Mitigating Spurious Correlations with Memorization-Guided Dataset De-Biasing}}
\date{}
\vspace{-0.05in}
\author{
Arda Fazla, Abolfazl Hashemi\thanks{Authors with the School of Electrical and Computer Engineering, Purdue University, West Lafayette, IN 47907, USA.}}
\begin{document}
\maketitle
%%%%%%%%%%%%%%%%%%%%%%%%%%%%%%%%%%%%%%%%%%%%%%%%%%%%%%%%%%%%%
%%%%%%%%%%%%%%%%%%%%%%%%%%%%%%%%%%%%%%%%%%%%%%%%%%%%%%%%%%%%%

\input{sec/abstract}
\input{sec/intro}

\input{sec/dist_feats}
\input{sec/model}
\input{sec/experiments}
\input{sec/conclusion}

% References
\bibliography{manuscript}

\newpage

\onecolumn

\setcounter{section}{0}
\setcounter{table}{3}
\setcounter{figure}{4}
\setcounter{assumption}{0}
\setcounter{theorem}{-1}
\setcounter{equation}{3}
\setcounter{algorithm}{2}

% ---------- APPENDIX / SUPPLEMENT ----------
\appendix
\section*{Appendix}
\addcontentsline{toc}{section}{Appendix}
\markboth{Appendix}{Appendix}
\startcontents[appendix]
\printcontents[appendix]{l}{1}{\setcounter{tocdepth}{3}}

\newpage
\input{appendix/sec_r}
\input{appendix/sec_a}
\input{appendix/sec_b}
\input{appendix/sec_c}
\input{appendix/sec_d}
\input{appendix/sec_e}
\input{appendix/sec_f}
\input{appendix/sec_g}
\input{appendix/sec_h}
\input{appendix/sec_rcm_anal}
\input{appendix/sec_i}

\end{document}

%% file: sec/abstract.tex
\begin{abstract}
Real-world datasets often contain spurious correlations that are not causally related to the target label. When such correlations dominate the majority of training samples, models tend to rely on them, leading to misclassification of minority samples that do not exhibit the same spurious patterns. While a potential approach is to select subsets of data to better represent the minority samples, this may require access to group labels, which are typically unknown. Furthermore, as we demonstrate, widely used sample scoring functions in the invariant subset or coreset selection literature largely depend on spurious features and therefore fail to accurately capture the importance or difficulty of core, causally relevant features. Accordingly, we propose to mitigate spurious correlations by developing a two-stage sample scoring function that disentangles the learning dynamics of core and spurious features and evaluates their difficulty separately. Based on our proposed metric, we introduce a new algorithm to find and prioritize informative samples both with and without spurious correlations. Extensive experiments demonstrate that a standard ERM model trained on our selected samples achieves superior performance compared to state-of-the-art debiasing techniques, while requiring as little as 10\% of the original training data.
\end{abstract}

%% file: sec/intro.tex
\section{Introduction}

\begin{figure*}[!t]
    \centering
    \includegraphics[width=0.8\textwidth]{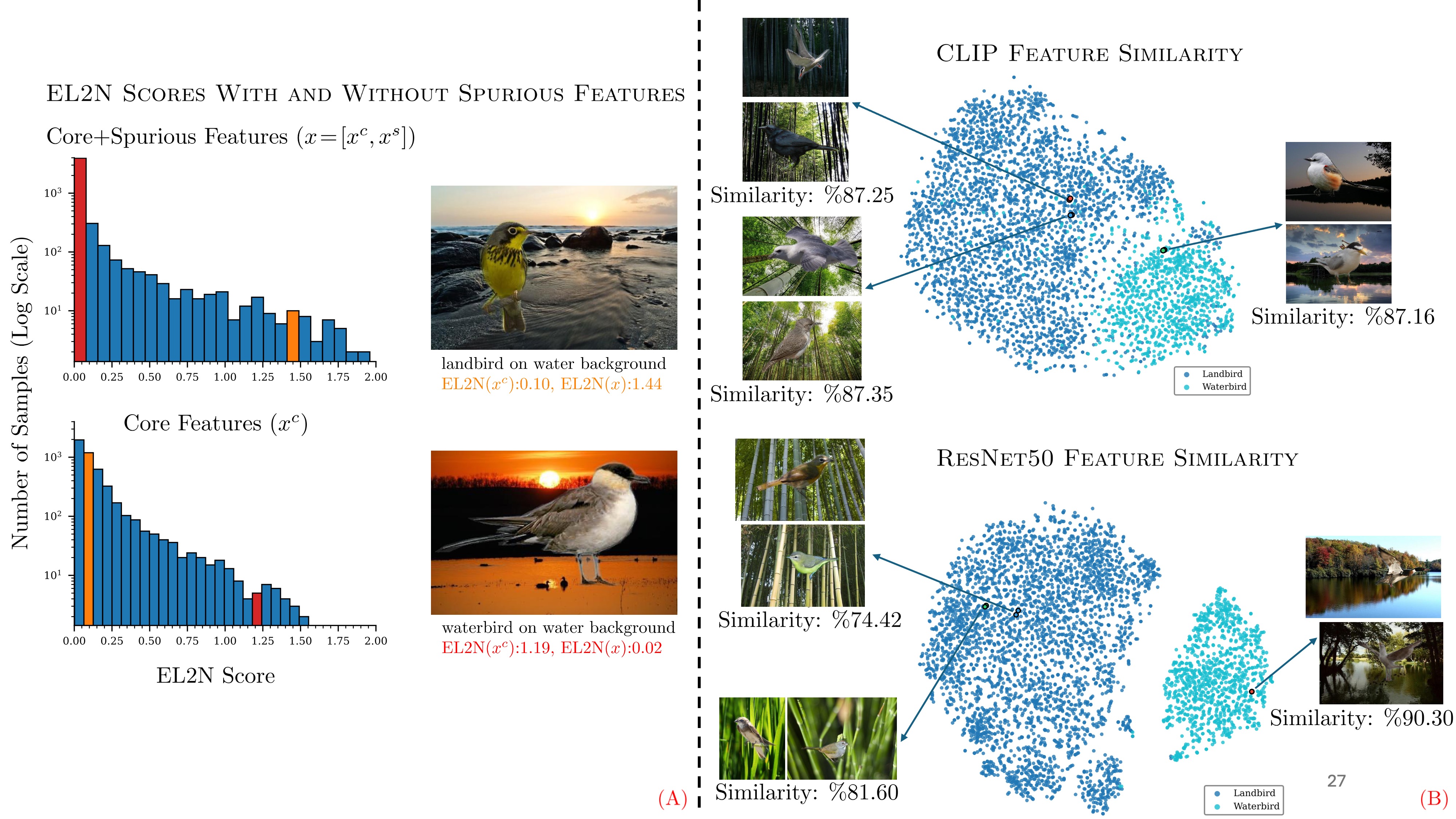}
    \caption{\textit{Comparison of standard sample scores on Waterbirds.} (A) We visualize the EL2N scores for two example images from the Waterbirds dataset, computed on a model trained on the datasets with and without the background. The results demonstrate that the presence of background significantly changes the EL2N score. (B) We present representative image pairs with high similarity based on feature embeddings extracted from a ResNet50 model trained on Waterbirds and from CLIP. In both cases, high similarity is primarily driven by shared background characteristics, even when the bird attributes differ substantially.}
    \label{fig:comp_methods}
\end{figure*}

\begin{figure*}[t]
    \centering    \includegraphics[width=0.8\textwidth]{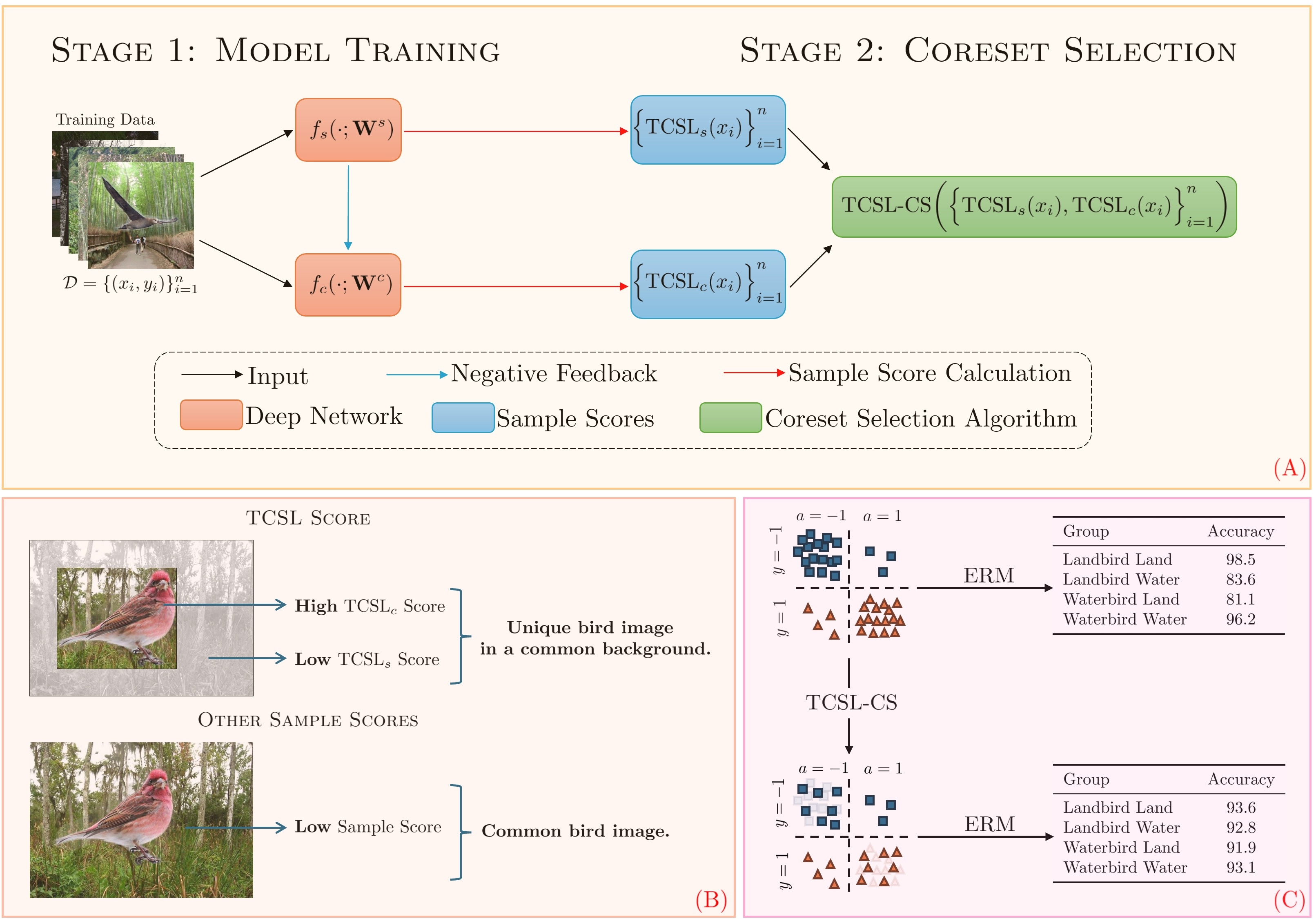}
    \caption{\textit{Overview of our proposed coreset selection algorithm.} (A) We first train a two-stage model to accurately disentangle the learning processes of spurious and core features and compute sample scores for each component separately ($\text{TCSL}_s$ and $\text{TCSL}_c$). We then construct our coreset selection algorithm, the Two-Stage Cumulative Sample Loss (TCSL)-guided Coreset Selection (TCSL-CS), based on the computed scores. (B) We illustrate the advantage of our TCSL score over existing sample scoring functions in the literature, showing its ability to disentangle the difficulty of different feature components within the image. (C) Training an ERM model on the coreset selected by TCSL-CS achieves state-of-the-art WGA on the Waterbirds dataset using only 10\% of the total training data.}
    \label{fig:main_figure}
\end{figure*}

Real-world datasets often contain a large number of samples with spurious correlations that are highly consistent within a class but not predictive of the true class label. Deep models learn to base their predictions on these simpler spurious features rather than the more complex core features \citep{SPARE, SIMP_BIAS_THEORY}, leading to poor worst-group accuracy on test data where the spurious correlations may not hold. Thus, a large body of recent work has focused on developing specialized algorithms to mitigate the bias arising from spurious correlations. However, these approaches deviate from standard ERM training, which remains widely used in practice. This raises the question: can we reduce a model’s reliance on spurious correlations while retaining standard ERM training, without introducing complex, specially designed optimization techniques?

One intuitive approach that could aid with addressing the above question is to leverage existing invariant data or sample selection algorithms designed to construct high-quality subsets of data capable of representing the full training distribution. Leveraging this idea, one may form an informative sample collection, or \emph{coreset}, that represents all groups within the dataset well.
%Another important line of work in deep learning focuses on developing effective coreset selection algorithms that construct high-quality subsets of data capable of representing the full training distribution without sacrificing from model accuracy or generalization. 
%One potential approach to addressing the above question is therefore to construct an informative coreset that adequately represents all groups within the dataset. 
Training a standard ERM model on such a coreset may then achieve competitive performance across all groups. However, we argue that commonly utilized coreset selection algorithms fail to construct strong coresets on spurious datasets, as they are not explicitly designed to ensure high worst-group accuracy, but rather to ensure high average test accuracy. As a result, certain groups within the dataset may be underrepresented in the selected coreset, even when the overall average accuracy remains high. Thus, these coreset selection methods are not directly applicable to spurious datasets. Recent works \citep{CS_LR, DROP} demonstrate that coreset selection strategies built on commonly used sample scores, such as EL2N \citep{EL2N} and SelfSup \citep{SelfSup}, fail to consistently achieve high worst-group accuracy across datasets with known spurious correlations.

Commonly used sample scoring functions and the coreset selection algorithms built upon them typically assign scores to individual samples to reflect their difficulty and derive selection strategies accordingly. We argue that, due to the ``simplicity bias'' phenomenon widely observed in deep learning~\citep{SIMP_BIAS_1, SIMP_BIAS_2, SIMP_BIAS_3, SIMP_BIAS_THEORY}, models tend to learn simple spurious features before capturing more complex core features. Consequently, the behavior of sample scoring functions that rely on model outputs or loss values becomes dominated by spurious features. As a result, samples without spurious correlations are often assigned high scores and classified as ``hard'', whereas samples with spurious correlations receive low scores and are treated as ``easy''. This bias causes commonly used scoring functions in coreset selection to inadequately capture the strength of the underlying core features. We theoretically analyze the impact of simplicity bias on the learning speeds of core and spurious features, as well as on loss-based sample scores, in Section~\ref{sec:problem_setting} and Appendix~\ref{sec:app_theo_anal}.

We illustrate the effects of simplicity bias in Figure~\ref{fig:comp_methods}. In Figure~\ref{fig:comp_methods}(A), we show that the EL2N score changes significantly when the background feature is removed from the dataset, indicating that a common bird image may receive a higher score than a unique bird image solely due to background characteristics. Beyond loss-based sample scoring functions, certain coreset selection algorithms also incorporate sample similarity based on the extracted feature embeddings. Accordingly, we show representative image pairs with high similarity from the Waterbirds dataset in Figure~\ref{fig:comp_methods}(B), identified using cosine similarity of feature embeddings extracted from a ResNet model and the foundation model CLIP. In both cases, similarity is primarily driven by shared background features, even when bird attributes differ substantially. Furthermore, a recent study~\citep{MS_DP} demonstrates that even a small number of samples with simple spurious features and complex core features can cause the model to rely predominantly on the spurious feature for prediction. Motivated by these observations, we argue that commonly used sample scoring functions are dominated by spurious features; consequently, coreset selection algorithms built upon them may fail to distinguish between easy and hard core features, resulting in suboptimal coresets and poor worst-group accuracy.

As a remedy, we propose the Two-Stage Cumulative Sample Loss (TCSL) and the TCSL-guided Coreset Selection (TCSL-CS) algorithm. As illustrated in Figure \ref{fig:main_figure}(A), our approach builds on widely used two-stage training methods from the literature on learning under spurious correlations \citep{LC, LFF, uLA, MAT} and distinguishes the score computation for core and spurious features without requiring access to the spurious attribute (group labels). Accordingly, TCSL consists of two scores, $\text{TCSL}_s$ and $\text{TCSL}_c$, representing the computed difficulty of the spurious and core features, respectively. As shown in Figure \ref{fig:main_figure}(B), while conventional sample scoring functions assign a single score to each image, which is dominated by the spurious feature (e.g., background), our TCSL framework separately evaluates the difficulty of the core (bird) and spurious (background) components. We then design our TCSL-CS algorithm based on TCSL to effectively select coresets that achieve both (1) high average accuracy and (2) high worst-group accuracy. As shown in Figure~\ref{fig:main_figure}(C), a coreset selected by TCSL-CS using only 10\% of the training data improves the worst-group accuracy of a standard ERM model by 11.33\% on the Waterbirds dataset, outperforming baselines that require group labels or complex optimization procedures.

Our contributions and scope: 
\begin{itemize}
    \item We propose the TCSL score, which separately quantifies the learning difficulty of core and spurious features in datasets with spurious correlations.
    \item Our proposed TCSL score enables us to leverage coreset selection as a principled tool to introduce TCSL-CS, an algorithm that selects effective coresets for datasets with spurious correlations, achieving both high average accuracy and high worst-group accuracy.
    \item We provide a strong theoretical analysis of the distinct learning dynamics of core and spurious features.
    \item Through extensive experiments on datasets with spurious correlations, we show that TCSL-CS outperforms existing debiasing and sample scoring baselines without requiring access to group labels.
\end{itemize}

%% file: sec/dist_feats.tex
\section{Problem Formulation}

\subsection{Problem Setting}\label{sec:problem_setting}

Let $\mathcal{D}\!=\!\{(x_1, y_1), \dots, (x_n, y_n)\}$ denote the training dataset of size $n$, where for each data sample we observe an input feature vector $x_i \in \mathbb{R}^d$ and its corresponding label $y_i$. For simplicity and without loss of generality, throughout the analysis, we focus on the binary classification setting where $y_i \in \{\pm1\}$. Each data sample is also associated with an unobserved spurious attribute $a_i \in \{\pm1\}$. We assume that within each class, the training data is partitioned into a majority group and a minority group, where the majority group contains all the samples with $a_i = y_i$, and the minority group contains all the samples with $a_i=-y_i$. The fraction of samples in the majority group within each class is denoted by $\alpha_y$, where $\alpha_y>0.5$ is commonly observed in practice. We denote the fraction of majority samples in the entire dataset by $\alpha$. We consider the setting in which the spurious attribute, and thus the group labels, is unknown to us. We assume each sample $x_i$ consists of core and spurious components, $x_i = [x_i^c, x_i^s]$, where $x_i^c$ is correlated with $y_i$ and $x_i^s$ is correlated with $a_i$. For example, in the Waterbirds~\citep{Waterbirds} dataset, the bird region serves as the core feature $x_i^c$, while the background serves as the spurious feature $x_i^s$.

Deep learning models are typically trained via empirical risk minimization (ERM), where given a model with probability outputs $f(x; \mathbf{W})$ and weights $\mathbf{W}$, we minimize $\mathcal{L}(\mathbf{W}) = \frac{1}{N} \sum_{i=1}^{N} \ell(y_i,f(x_i;\mathbf{W}))$ where $\ell$ can be any suitable loss function, e.g., cross entropy loss or sigmoid loss. Commonly, gradient-based optimization techniques are utilized, such as stochastic gradient descent (SGD), where the parameters of the model are updated at each iteration as $ \mathbf{W}_{t+1} = \mathbf{W}_{t} - \eta_t \nabla \mathcal{L}(\mathcal{B}_t, \mathbf{W}_{t}),$
where $\eta_t$ denotes the learning rate at iteration $t$, and $\mathcal{B}_t$ is the mini-batch of samples from $\mathcal{D}$ used at step $t$. Depending on the context, we use $t$ interchangeably to denote both the update and epoch step, with $T$ denoting the total number of epochs. Here, $\nabla \mathcal{L}(\cdot, \cdot)$ represents the stochastic gradient of batch loss $\mathcal{L}(\mathcal{B}_t, \mathbf{W}_{t})$, which is defined as the average weighted loss over batch:
\begin{equation}\label{eq:mini_batch_loss}
    \mathcal{L}(\mathcal{B}_t, \mathbf{W}_{t}) =  \sum_{i=1}^{|\mathcal{B}_t|} \lambda_i \, \ell(y_i,f(x_i;\mathbf{W}_t)),
\end{equation}
where $\lambda_i$ denotes the weight associated with the $i$-th sample in the batch. Typically, $\lambda_i$ is set to $1/|\mathcal{B}_t|$ to ensure uniform averaging. In the presence of class imbalanced data, $\lambda_i$ is often chosen as the inverse number of samples belonging to the same class, which helps prevent mode collapse, i.e., the model degenerating to predict only the majority class. To maintain consistency across batches, an additional normalization step is applied such that the sum of sample weights within each batch equals one, ensuring that the overall optimization problem remains unchanged.

To evaluate algorithm performance, we use average test accuracy (ACC), defined as
\[
\mathrm{ACC}(\mathbf{W}) =\mathbb{P}_{(x, y) \sim \mathcal{D}_{\text{test}}} \big[\argmax(f(x;\mathbf{W}))=y\big].
\]
With spurious correlations, researchers are often more interested in the worst-group accuracy (WGA):
\[
\mathrm{WGA}(\mathbf{W})\!=\!
\min_{\substack{y \in \{\pm 1\} \\ a \in \{\pm 1\}}}
\mathbb{P}_{(x, y, a) \sim \mathcal{D}_{\text{test}}}
\big[\argmax(f(x;\mathbf{W}))\!=\!y\big],
\]
which measures the worst accuracy of the model among all groups defined by combinations of $y$ and $a$.

Our overall goal is, given a coreset selection ratio $r$, to select $r|\mathcal{D}|$ samples from the training dataset $\mathcal{D}$ that maximize WGA on an unobserved test set. We consider a coreset successful if a standard ERM model trained on it achieves optimal WGA. Our objective is therefore to eliminate the dataset bias induced by spurious correlations and hence improve WGA, without compromising overall generalization performance. Equivalently, the task can be viewed as selecting a coreset of size $c|\mathcal{D}|$ that effectively and equally represents all groups in the training dataset, without access to the group labels.

\subsection{Feature Learning Analysis}\label{sec:prob_form}

\begin{figure}[t]
\centering
\includegraphics[width=0.4\columnwidth]{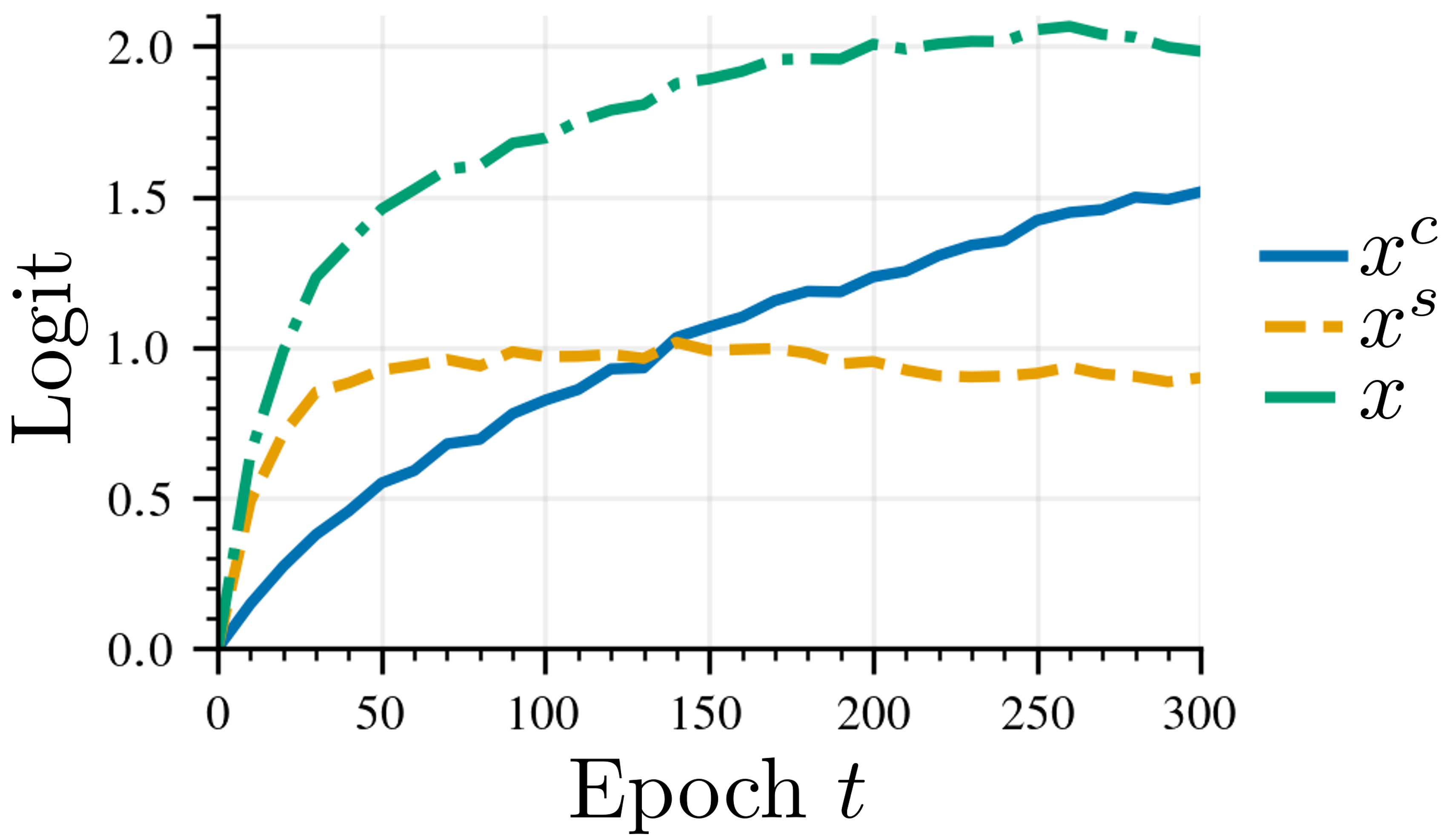}
\vspace{-4pt}
\caption{Average logits computed over the entire dataset. $x$, $x^c$, and $x^s$ denote the logits from full images, core features only and spurious features only, respectively.}
\vspace{-6pt}
\label{fig:core_spur_logit}
\end{figure}

We begin by demonstrating the logit outputs of a ResNet50 model trained with SGD on the Waterbirds dataset, as shown in Figure \ref{fig:core_spur_logit}. Every few epochs, we record the average model predictions for training samples belonging to the largest majority group: landbirds on land backgrounds. The results are collected for the full image, as well as for its core and spurious components, which are manually extracted and denoted as $x$, $x^c$, and $x^s$, respectively. As illustrated in Figure \ref{fig:core_spur_logit}, the logits induced by the spurious features converge much faster than those induced by the core features. This commonly observed behavior in deep learning is called the simplicity bias phenomenon, where models exhibit a strong preference for simpler features over more complex ones \citep{SIMP_BIAS_1, SIMP_BIAS_2, SIMP_BIAS_3}. In our context, when the spurious feature is frequently observed in the dataset and has a higher feature strength than its core counterpart, the model tends to ignore the core feature during the early stages of training, rapidly fitting to the spurious correlations instead~\citep{SPARE, COMP_MATTERS}. While prior work quantitatively defines feature strength in various ways~\citep{SPARE, NTK}, we avoid setting a strict definition and instead treat the strengths of the spurious and core features as unknown, independent, non-negative quantities denoted by $\beta^s$ and $\beta^c$, respectively.

In Theorem~\ref{thm:thm_1} (proved in Appendix \ref{sec:homo-spiked}), we analyze the learning dynamics of a spiked covariance model~\citep{SCM} under the Neural Tangent Kernel (NTK) regime~\citep{NTK_1} and provide theoretical justification showing that the learning speed of spurious features dominates that of core features during the early stages of training. We focus on the NTK regime (see Appendix \ref{sec:app_theo_anal} for details), following similar theoretical studies \citep{NTK, SPARE} as it offers a mathematically tractable and strong approximation of gradient-based training, which has been shown to capture the learning behavior of a broad class of deep learning architectures \citep{NTK_1, NTK_2}. We adopt the spiked covariance model, as it enables a clean separation between core and spurious features while also reflecting key properties of real-world datasets \citep{SCM_1}.

% Theorem 1

\begin{theorem}\label{thm:thm_1}
    Define the simplicity bias condition as $\beta^s (2\alpha - 1) > \beta^c$ and assume this holds. Let $\bar{m}_t(x_i) = \mathbb{E}_{\mathbf{W}_0}[y_i h(x_i; \mathbf{W}_t)]$ be the expected margin of the model.
Then, there exists a time $T' > 0$ such that for all $t \in (0, T')$:
\begin{enumerate}
    \item For any sample from the majority group with  $a_i=y_i$, the expected margin $\bar{m}_t(x_i)$ is positive, and the loss $\ell(\bar{m}_t(x_i))$ is less than $\log(2)$.
    \item For any sample from the minority group with $a_i=-y_i$, the expected margin $\bar{m}_t(x_i)$ is negative, and the loss $\ell(\bar{m}_t(x_i))$ is greater than $\log(2)$.
\end{enumerate}
\end{theorem}

\noindent We analyze the expected initial growth rates of the true core margin and the true spurious margin, and show that under the simplicity bias condition $\beta^s (2\alpha - 1) > \beta^c$, the spurious subnetwork is learned faster than the core subnetwork during the initial stages of training. Theorem \ref{thm:thm_1} and Theorem \ref{thm:homo_acceleration} on initial acceleration (stated in Appendix \ref{sec:homo-spiked} due to space limitations) which help explain the phenomenon demonstrated in Figure~\ref{fig:core_spur_logit} are further generalized in Appendix \ref{sec:het-spiked}.

\noindent \textbf{Takeaway.} In the proof of Theorem \ref{thm:thm_1}, we show that the learning of the spurious feature will be faster than the core feature during the early stages of training. Hence, the model output will be dominated by the spurious subnetwork. Consequently, there exists $T' > 0$ such that, for all $t \in (0, T')$, the expected loss of the majority group is upper bounded by $\log(2)$, while the expected loss of the minority group is lower bounded by $\log(2)$ over the same interval. Together, these results indicate that sample losses are shifted to smaller or larger values depending on the spurious attribute, independent of the strength of the core features.

Building on Theorem~\ref{thm:thm_1}, we next show that widely used sample scoring functions in the coreset selection literature cannot be directly employed in the presence of spurious correlations, and motivate the need to separately quantify the difficulty of the core and spurious components of each sample.

\noindent \textbf{Takeaway 2.} Consider any sample scoring function that is non-decreasing in the per-step loss, which captures a wide variety of sample scores employed in the coreset selection literature, e.g., instantaneous loss, average/cumulative loss, area under margin, EL2N, GraNd. Since the spurious feature is learned faster than the core feature, the effect of the spurious component will be evident throughout training. Consequently, by Theorem~\ref{thm:thm_1}, the loss (therefore the assigned sample score) will be dominated by the spurious feature and the (unknown) spurious attribute $a$. Since widely used sample scoring functions~\citep{EL2N, SelfSup, Memorization} assign a single score to each sample, they are inherently unable to assess the feature strengths of the core and spurious components simultaneously. We extend and validate our theoretical findings under the NTK regime to more general deep learning settings through experiments in Section~\ref{sec:exp}. 

%We now examine commonly used sample scoring functions in coreset selection and analyze their behavior under spurious correlations. These functions assign scores to individual samples based on their computed difficulty: strong features corresponding to easy and simple samples receive low scores, while weak features corresponding to hard and complex samples receive high scores~\citep{CS_LR}. However, under spurious correlations, each sample $x$ can be decomposed into core and spurious components whose strengths vary independently. Since widely used sample scoring functions~\citep{EL2N, SelfSup, Memorization} assign a single score to each sample, they are inherently unable to assess the feature strengths of the core and spurious components simultaneously.

%% file: sec/model.tex
\section{TCSL-CS}\label{sec:model}
Our goal is to construct two separate subnetworks to disentangle the learning dynamics of spurious and core components. These subnetworks enable us to independently compute difficulty scores associated with spurious and core features. In Section~\ref{sec:model_train}, we introduce a two-stage network designed to separate the learning processes of the spurious and core components. In Section~\ref{sec:memo_proxy}, we propose a two-stage sample scoring function based on the learned subnetworks. Finally, in Section~\ref{sec:cor_sel}, we present \textbf{Algorithm~\ref{alg:tcsl_coreset}} that integrates information from both core and spurious components to achieve group-robust coreset selection.

\subsection{Model Training}\label{sec:model_train}
To separate the sample scores into spurious and core components, we propose a two-stage model training scheme described in \textbf{Algorithm~\ref{alg:two_stage}}. For notational simplicity, we write
$f_s(x_i) = f_s(x_i;\mathbf{W}_t^s)$ and
$f_c(x_i) = f_c(x_i;\mathbf{W}_t^c)$
and omit the explicit dependence on the parameters whenever it is clear from the context. In the first stage, we train a biased model designed to learn only the spurious features. As supported by our theoretical analysis in Appendix \ref{sec:app_theo_anal} and by prior work on simplicity bias \citep{SPARE, PDE, COMP_MATTERS}, spurious features are learned earlier in training compared to core features. Consequently, samples with strong and easy spurious features are learned the fastest \citep{PDE}. Building on this intuition, at the end of each epoch, we upweight samples with the lowest loss so that they contribute more to the learning process during the next epoch. Hence, we effectively upweight samples for which the spurious attribute agrees with the label ($a = y$), thereby amplifying the learning process of the bias in the data. The proposed architecture is an adapted version of the FLOW algorithm~\citep{FLOW}, originally developed for multi-task learning, which we modify for our task. The median term in the denominator is used to control the range of sample weights across iterations.

\begin{algorithm}[t]
\caption{Two-Stage Model Training}
\label{alg:two_stage}
\textbf{Input:} Dataset $\mathcal{D} = \{(x_i, y_i)\}_{i=1}^n$, training epochs $T_c$ and $T_s$. Let $n_{y_i}$ denote the number of samples in class $y_i$.\\
\textbf{Initialize:} 
\begin{itemize}
    \item Classifiers $f_c$, $f_s$
    \item Sample weights $\lambda_i = \frac{1}{n_{y_i}}$ for all $i \in \{1, \ldots, n\}$ (uniform within each class)
\end{itemize}
\textbf{Output:} Trained models $f_c$, $f_s$.
\begin{algorithmic}[1]
\State \textbf{Stage 1: Train spurious model $f_s$}
    \For{$t = 1$ to $T_s$}
        \State Train $f_s$ on $\mathcal{D}$ using SGD according to Equation~\eqref{eq:mini_batch_loss}
        \State Update sample weights:
        \[
        \lambda_i \gets \exp\!\left(- \frac{\ell(y_i, f_s(x_i))}{\operatorname{median}_j\big(\ell(y_j, f_s(x_j))\big)}\right) \quad \forall i
        \]
    \EndFor
    \State Freeze $f_s$ and reset $\lambda_i = \frac{1}{n_{y_i}}$ for all $i$
\Statex
\State \textbf{Stage 2: Train core model $f_c$}
    \For{$t = 1$ to $T_c$}
        \State Train $f_c$ on $\mathcal{D}$ using SGD according to Equation~\eqref{eq:core_loss}
    \EndFor
\end{algorithmic}
\end{algorithm}

After training the spurious (biased) model $f_s$ for $T_s$ epochs, we freeze its weights and reset the sample weights $\lambda_i$. We then initialize the core model $f_c$ and train it for $T_c$ epochs using SGD, where the cross-entropy loss over a random mini-batch $\mathcal{B}_t$ is adjusted to
\begin{align}
    \mathcal{L}(\mathcal{B}_t, \mathbf{W}_{t}^c) 
    = \sum_{i=1}^{|\mathcal{B}_t|} \lambda_i \, \ell\bigg( 
    &y_i, f_c(x_i;\mathbf{W}_{t}^c) + \log\!\big(f_s(x_i)\big)\bigg).
    \label{eq:core_loss}
\end{align}
Only the core model $f_c$ is updated through Equation~\eqref{eq:core_loss}, as $f_s$ is frozen. Since $f_s$ outputs probabilities,  we have $f_s(x) \in [0,1]$. Intuitively, applying the $\log$ operator yields negative values that reflect the strength of the learned spurious relationships captured by $f_s$. Consequently, $f_s$ acts as a negative feedback mechanism and $f_c$ is pushed towards learning the core features.

Similar two-stage debiasing architectures have been proposed in the literature under the name ``logit correction''~\citep{LC,uLA,LFF,MAT}. In contrast to prior work, we do not employ a debiasing architecture to optimize model performance, but rather to compute sample scores associated exclusively with the spurious and core features, which are then used for coreset selection.

\subsection{Two-Stage Cumulative Sample Loss}\label{sec:memo_proxy}
For our two-stage per-sample score, we build on the literature on the memorization score, initially proposed by~\cite{Memorization}. Since computing the memorization score is computationally inefficient, we instead adopt its recently proposed proxy, the Cumulative Sample Loss (CSL)~\citep{CSL}, where the authors show that the originally proposed memorization score is bounded by the accumulated sample losses over epochs. The original definitions of these scores are provided in Appendix \ref{sec:app_memo_csl}.
We compute the CSL scores for the two learned networks separately. Our Two-Stage Cumulative Sample Loss (TCSL) is defined as:
\begin{equation}
    \text{TCSL}(x) = \left[\text{TCSL}_s(x), \text{TCSL}_c(x)\right],
\end{equation}
where $\text{TCSL}_s(x) = \frac{1}{T_s}\sum_{t=1}^{T_s}\ell(y,f_s(x;\mathbf{W}_t^s))$ and $\text{TCSL}_c(x) = \frac{1}{T_c}\sum_{t=1}^{T_c}\ell(y,f_c(x;\mathbf{W}_t^c))$. As TCSL requires sample losses after every epoch, we compute and store the loss of each samples during subnetwork training in $\textbf{Algorithm \ref{alg:two_stage}}$. Intuitively, for samples where the spurious attribute disagrees with the label ($a=-y$), the $\text{TCSL}_s$ score should be high, as the spurious model $f_s$ is likely to misclassify such samples. In contrast, the $\text{TCSL}_c$ score should depend solely on the difficulty of the core features. We verify these observations experimentally in Section~\ref{sec:exp}.

\subsection{Group Robust Coreset Selection}\label{sec:cor_sel}

%==================  Main: TCSL-guided Coreset Selection  =====================
\begin{algorithm}[t]
\caption{TCSL-guided Coreset Selection}
\label{alg:tcsl_coreset}
\textbf{Input:} TCSL scores for $\mathcal{D}=\{(x_i,y_i)\}_{i=1}^{n}$, selection ratio $r$, loss curves of $f_s$ for every sample, number of bins $B$, threshold $\tau$.\\
\textbf{Output:} Coreset $\mathcal{D}'$ of size $r|\mathcal{D}|$.
\begin{algorithmic}[1]
\State \textbf{Stage 1: Identify groups (by TCSL$_s$)}
\State Set $\mathrm{w}_i = \text{TCSL}_s(x_i), \;
\mathrm{L}_{i,t} = \ell\big(y_i, f_s(x_i;\mathbf{W}_t^s)\big)$
\State $\{G_{\text{high}}, G_{\text{low}}\} \gets \textproc{wKMeans}(\mathbf{w}, \mathbf{L})$
\State $\mathcal{D}' \gets \textproc{UnifRandSelect}\!\left(G_{\text{high}},\, \min\!\big(|G_{\text{high}}|,\, r|\mathcal{D}|\big)\right)$
\Statex
\State \textbf{Stage 2: Fill remaining quota (by TCSL$_c$)}
\If{$|\mathcal{D}'| < r|\mathcal{D}|$}
    \State $G_{\text{low}} \gets \textproc{Sort}\big(G_{\text{low}};\ \text{ascending by TCSL}_c\big)$
    \State $\tilde{n} \gets r|\mathcal{D}| - |\mathcal{D}'|$
    \If{$r \ge \tau$}
        \State $S_{\text{low}} \gets \textproc{SelectBot}(G_{\text{low}},\, \tilde{n})$
    \Else
        \State $S_{\text{low}} \gets \textproc{SelectHist}(G_{\text{low}},\, \tilde{n},\, B)$
    \EndIf
    \State $\mathcal{D}' \gets \mathcal{D}' \cup S_{\text{low}}$
\EndIf
\State \textbf{Return} $\mathcal{D}'$
\end{algorithmic}
\end{algorithm}
\vspace{-3pt}
We propose a group-robust coreset selection algorithm based on the TCSL score, termed TCSL-guided Coreset Selection (TCSL-CS), described in \textbf{Algorithm~\ref{alg:tcsl_coreset}}. We first identify samples with high $\text{TCSL}_s$ scores (samples without strong spurious features) using the $\textproc{wKMeans}$ algorithm on the $\text{TCSL}_s$ scores and the loss curves of $f_s$. $\textproc{wKMeans}$ applies a weighted variant of the k-means algorithm~\citep{KMEANS}, where each sample is represented by its losses computed by $f_s$ over $T_s$ epochs and each sample is assigned a weight given by its $\text{TCSL}_s$ score. Since the sample losses are already obtained during the computation of $\text{TCSL}_s$, this step introduces no additional computational cost but provides the $\textproc{wKMeans}$ algorithm with additional information for clustering. Details of $\textproc{wKMeans}$ are provided in Appendix \ref{sec:app_helper_algs}. Samples with high $\text{TCSL}_s$ scores are grouped into $G_{\text{high}}$, from which we randomly select samples until the quota $\min\!\big(|G_{\text{high}}|,\,\!r|\mathcal{D}|\big)$ is filled. As $\alpha\!>\!0.5$ by definition, the majority of samples are assigned to group $G_{\text{low}}$.

Next, we fill the remaining quota ($r|\mathcal{D}| - |\mathcal{D}'|$) with samples from $G_{\text{low}}$. Following recent consensus in the coreset selection literature, we follow two different selection strategies depending on the selection ratio $r$. Since the simplicity bias arises primarily from samples with simple spurious but complex core features, our goal is to eliminate such samples. Therefore, for $r \geq \tau$, we use $\textproc{SelectBot}$ to select samples with low $\text{TCSL}_c$ scores. However, when the coreset selection ratio is small ($r < \tau$), selecting only samples with low $\text{TCSL}_c$ scores reduces diversity and is thus suboptimal. In this case, we use $\textproc{SelectHist}$, which partitions the samples into $B$ bins based on the $\text{TCSL}_c$ scores such that each bin contains an equal number of samples and randomly samples from each bin until the quota is filled. Finally, we combine the samples selected from both stages to form the final coreset $\mathcal{D}'$. Details of $\textproc{SelectBot}$ and $\textproc{SelectHist}$ are provided in Appendix \ref{sec:app_helper_algs}. We provide the runtime, computational complexity, and memory usage analysis of our algorithm in Appendix \ref{sec:app_rcsm}.

%% file: sec/experiments.tex
\section{Experiments}\label{sec:exp}

\subsection{Performance Comparison}

In this section, we evaluate the effectiveness of the proposed TCSL-CS algorithm on four computer vision benchmark datasets known to exhibit spurious correlations: Waterbirds~\citep{Waterbirds}, cMNIST~\citep{Cmnist}, MetaShift~\citep{Metashift}, and UrbanCars-B~\citep{Urbancars}. Additional details on the datasets are provided in Appendix \ref{sec:app_dataset_details} and training details for each dataset are provided in Appendix \ref{sec:app_hyperparams}.

First, we evaluate TCSL-CS as a debiasing algorithm by comparing the performance of an ERM model trained on the coreset identified by TCSL-CS with coreset selection ratio $r\!=\!0.1$ against state-of-the-art debiasing methods from the literature. We choose this ratio to assess the effectiveness of our approach, as it highlights the two key capabilities of our proposed method: (1) clustering samples into high and low spuriosity regions as identified by the $\text{TCSL}_s$ score and (2) ranking high spuriosity samples based on their core feature difficulty, as captured by the $\text{TCSL}_c$ score. Results for different coreset selection ratios are presented in Appendix \ref{sec:app_add_exp}. Additional information on the baseline methods is provided in Appendix \ref{sec:app_info_baseline}.

Table~\ref{tab:tab_1_1} and Table~\ref{tab:tab_1_2} present the worst-group accuracy (WGA) and average accuracy (AVG) on the test datasets for a standard ERM model trained on the coreset identified by our TCSL-CS algorithm, compared with other baseline debiasing methods. Training a standard ERM model on the TCSL-CS coreset substantially improves WGA, achieving performance on par with state-of-the-art baselines across multiple datasets, even outperforming models that have access to group labels (spurious attribute) in some cases. For instance, TCSL-CS improves the WGA of ERM by 10.44\% on the Waterbirds dataset, outperforming GroupDRO, a strong baseline that requires access to the spurious attribute.

Next, we compare the performance of TCSL-CS against other coreset selection algorithms from the literature. To the best of our knowledge, TCSL-CS is the first coreset selection method specifically designed to operate under spurious correlations and achieve group robustness. For completeness, we follow the experimental setup of recent work~\citep{CS_LR} and use EL2N, Random, RGbal, and SelfSup as our baselines for computing sample scores. We then construct coreset selection strategies under three different settings: (Bot) selects samples with the lowest scores, (Top) selects samples with the highest scores and (Hist) selects samples using a histogram-based approach. We also provide four versions of the state-of-the-art D2 coreset selection algorithm \citep{D2}, which effectively combines sample scores with feature similarity for enhanced data selection and coverage. Further details for these algorithms are provided in Appendix \ref{sec:app_info_baseline}.

The results are shown in Table~\ref{tab:tab_2}. Overall, TCSL-CS outperforms all other methods across all datasets, except on Waterbirds, where RGbal achieves slightly better performance with access to group labels. Moreover, on the cMNIST dataset, where spurious correlations are particularly strong ($\alpha = 0.995$), TCSL-CS surpasses the next best method by more than 11\%. The poor performance of baseline sample scores supports our theoretical analysis in Section~\ref{sec:prob_form} and aligns with recent findings~\citep{CS_LR} showing that widely used coreset selection algorithms perform poorly when spurious correlations are dominant in the data. Additional results for different coreset selection ratios are provided in Appendix \ref{sec:app_add_exp}.

\begin{table}[t]
\centering
\small
\setlength{\tabcolsep}{3pt} % reduce column spacing
\resizebox{0.8\linewidth}{!}{%
\begin{tabular}{@{}lcccc@{}}
\toprule
\textbf{Method} & \multicolumn{2}{c}{\textbf{Group Info}} & \textbf{Waterbirds} & \textbf{cMNIST} \\
& Train & Val & WGA (AVG) & WGA (AVG) \\
\midrule
CB ERM             & x & x & $81.15{\scriptsize\pm2.26}$ ($97.97{\scriptsize\pm0.03}$) & $56.61{\scriptsize\pm3.41}$ ($76.01{\scriptsize\pm2.54}$) \\
GB ERM             & \checkmark & \checkmark & $89.98{\scriptsize\pm0.77}$ ($98.20{\scriptsize\pm0.03}$) & $\underline{81.46}{\scriptsize\pm3.33}$ ($91.83{\scriptsize\pm1.10}$) \\
GroupDRO*          & \checkmark & \checkmark & $91.40{\scriptsize\pm1.10}$ ($93.50{\scriptsize\pm0.30}$) & $58.66{\scriptsize\pm5.46}$ ($76.56{\scriptsize\pm2.26}$) \\
LC*                & x & \checkmark & $90.5{\scriptsize\pm1.1}$ & $71.25{\scriptsize\pm3.17}$ \\
DFR*               & x & \checkmark & $\boldsymbol{92.1}$ ($96.7$) & $74.2$ ($93.7$) \\
CNC*               & x & \checkmark & $88.5{\scriptsize\pm0.3}$ ($90.9{\scriptsize\pm0.1}$) & $77.4{\scriptsize\pm3.0}$ ($90.9{\scriptsize\pm0.6}$) \\
LfF*               & x & \checkmark & $75.2$ ($97.5$) & $77.0$ \\
JTT*               & x & \checkmark & $86.0$ ($93.6$) & $74.04{\scriptsize\pm1.33}$ \\
ULA*               & x & \checkmark & $86.1{\scriptsize\pm1.5}$ ($91.5{\scriptsize\pm0.7}$) & $75.13{\scriptsize\pm0.78}$ \\
EIIL*              & x & x & $77.2{\scriptsize\pm1.0}$ ($96.5{\scriptsize\pm0.2}$) & $72.8{\scriptsize\pm6.8}$ ($90.7{\scriptsize\pm0.9}$) \\
GEORGE*            & x & x & $76.2{\scriptsize\pm2.0}$ ($95.7{\scriptsize\pm0.5}$) & $76.4{\scriptsize\pm2.3}$ ($89.5{\scriptsize\pm0.3}$) \\
TCSL-CS (ERM)      & x & x & $\underline{91.91}{\scriptsize\pm0.35}$ ($92.83{\scriptsize\pm0.55}$) & $\boldsymbol{83.37}{\scriptsize\pm1.33}$ ($91.76{\scriptsize\pm0.44}$) \\
\bottomrule
\end{tabular}
}
\caption{We compare the WGA of different methods on the Waterbirds and cMNIST datasets. We provide the AVG in parentheses, if available. * indicates the original scores reported by the authors or subsequent work. The best WGA for each dataset is shown in \textbf{bold} and the second best value is \underline{underlined}.}
\label{tab:tab_1_1}
\end{table}

\begin{table}[t]
\centering
\small
\setlength{\tabcolsep}{3pt} % reduce column spacing
\resizebox{0.8\linewidth}{!}{%
\begin{tabular}{@{}lcccc@{}}
\toprule
\textbf{Method} & \multicolumn{2}{c}{\textbf{Group Info}} & \textbf{MetaShift} & \textbf{UrbanCars-B} \\
& Train & Val & WGA (AVG) & WGA (AVG) \\
\midrule
CB ERM      & x & x & $73.15{\scriptsize\pm2.63}$ ($87.27{\scriptsize\pm0.27}$) & $66.00{\scriptsize\pm2.80}$ ($83.30{\scriptsize\pm0.36}$) \\
GB ERM      & \checkmark & \checkmark & $75.69{\scriptsize\pm2.50}$ ($88.48{\scriptsize\pm0.50}$) & $74.53{\scriptsize\pm3.63}$ ($87.23{\scriptsize\pm1.17}$) \\
GroupDRO    & \checkmark & \checkmark & $74.07{\scriptsize\pm2.89}$ ($87.62{\scriptsize\pm0.20}$) & $71.33{\scriptsize\pm1.29}$ ($85.63{\scriptsize\pm0.68}$) \\
LC          & x & \checkmark & $\underline{76.16}{\scriptsize\pm1.75}$ ($88.71{\scriptsize\pm0.63}$) & $\underline{78.67}{\scriptsize\pm1.29}$ ($87.40{\scriptsize\pm0.4}$) \\
DFR         & x & \checkmark & $75.23{\scriptsize\pm0.80}$ ($95.17{\scriptsize\pm0.40}$) & $77.50{\scriptsize\pm1.10}$ ($81.00{\scriptsize\pm0.40}$) \\
TCSL-CS (ERM) & x & x & $\boldsymbol{79.40}{\scriptsize\pm2.23}$ ($84.95{\scriptsize\pm1.52}$) & $\boldsymbol{84.27}{\scriptsize\pm1.67}$ ($86.93{\scriptsize\pm0.78}$) \\
\bottomrule
\end{tabular}
}
\caption{We compare the WGA of different methods on the MetaShift and UrbanCars-B datasets. We also provide the AVG in parentheses. Results are averaged over $3$ seeds. The best WGA for each dataset is shown in \textbf{bold} and the second best value is \underline{underlined}.}
\label{tab:tab_1_2}
\end{table}

\begin{table*}[!hbtp]
\centering
\resizebox{\textwidth}{!}{
\begin{tabular}{lccccc}
\toprule
\textbf{Method} & {\textbf{Group Info}} & \textbf{Waterbirds} & \textbf{cMNIST} & \textbf{MetaShift} & \textbf{UrbanCars-B} \\
& Train & WGA (AVG) & WGA (AVG) & WGA (AVG) & WGA (AVG) \\
\midrule
EL2N (Bot)     & x & $39.88${\scriptsize$\pm0.93$} ($93.90${\scriptsize$\pm0.16$}) & $0.00${\scriptsize$\pm0.00$} ($15.90${\scriptsize$\pm1.08$}) & $53.47${\scriptsize$\pm4.81$} ($78.76${\scriptsize$\pm2.01$}) & $35.60${\scriptsize$\pm4.33$} ($68.07${\scriptsize$\pm3.59$}) \\
EL2N (Top)     & x & $88.13${\scriptsize$\pm3.32$} ($89.53${\scriptsize$\pm2.61$}) & $52.67${\scriptsize$\pm5.84$} ($76.58${\scriptsize$\pm2.34$}) & $35.19${\scriptsize$\pm19.60$} ($59.61${\scriptsize$\pm9.74$}) & $28.53${\scriptsize$\pm4.74$} ($56.63${\scriptsize$\pm2.43$}) \\
EL2N (Hist)    & x & $64.94${\scriptsize$\pm9.30$} ($95.97${\scriptsize$\pm0.33$}) & $1.72${\scriptsize$\pm1.12$} ($24.10${\scriptsize$\pm5.07$}) & $64.12${\scriptsize$\pm4.93$} ($81.48${\scriptsize$\pm2.27$}) & $54.00${\scriptsize$\pm1.74$} ($75.33${\scriptsize$\pm0.61$}) \\
SelfSup (Bot)       & x & $59.81${\scriptsize$\pm0.52$} ($91.01${\scriptsize$\pm0.59$}) & $0.00${\scriptsize$\pm0.00$} ($18.48${\scriptsize$\pm0.26$}) & $38.19${\scriptsize$\pm6.01$} ($77.14${\scriptsize$\pm1.90$}) & $38.80${\scriptsize$\pm1.20$} ($74.90${\scriptsize$\pm0.46$}) \\
SelfSup (Top)       & x & $75.39${\scriptsize$\pm0.82$} ($96.25${\scriptsize$\pm0.28$}) & $58.15${\scriptsize$\pm3.09$} ($77.45${\scriptsize$\pm1.44$}) & $37.96${\scriptsize$\pm5.31$} ($75.35${\scriptsize$\pm3.01$}) & $27.47${\scriptsize$\pm5.72$} ($64.10${\scriptsize$\pm1.87$}) \\
SelfSup (Hist)      & x & $74.97${\scriptsize$\pm0.64$} ($96.11${\scriptsize$\pm0.21$}) & $3.94${\scriptsize$\pm3.72$} ($24.57${\scriptsize$\pm2.83$}) & $68.06${\scriptsize$\pm3.67$} ($84.26${\scriptsize$\pm1.28$}) & $57.87${\scriptsize$\pm10.51$} ($79.00${\scriptsize$\pm3.22$}) \\
D2 (EL2N+ResNet) & x & $89.92${\scriptsize$\pm2.14$} ($91.44${\scriptsize$\pm2.55$}) & $31.77${\scriptsize$\pm11.91$} ($58.74${\scriptsize$\pm4.54$}) & $26.85${\scriptsize$\pm5.57$} ($55.61${\scriptsize$\pm4.76$}) & $53.20${\scriptsize$\pm4.00$} ($66.27${\scriptsize$\pm0.78$}) \\
D2 (SelfSup+ResNet) & x & $75.49${\scriptsize$\pm1.67$} ($96.37${\scriptsize$\pm0.08$}) & $22.89${\scriptsize$\pm10.66$} ($55.98${\scriptsize$\pm5.44$}) & $37.50${\scriptsize$\pm11.47$} ($75.17${\scriptsize$\pm2.73$}) & $30.13${\scriptsize$\pm3.80$} ($68.77${\scriptsize$\pm1.14$}) \\
D2 (EL2N+CLIP) & x & $90.39${\scriptsize$\pm0.74$} ($96.82${\scriptsize$\pm0.37$}) & $28.18${\scriptsize$\pm10.38$} ($58.19${\scriptsize$\pm1.76$}) & $\underline{78.01}${\scriptsize$\pm0.40$} ($82.47${\scriptsize$\pm1.67$}) & $58.40${\scriptsize$\pm0.40$} ($65.13${\scriptsize$\pm3.17$}) \\
D2 (SelfSup+CLIP) & x & $75.49${\scriptsize$\pm1.41$} ($97.31${\scriptsize$\pm0.07$}) & $32.39${\scriptsize$\pm6.14$} ($61.57${\scriptsize$\pm2.31$}) & $46.53${\scriptsize$\pm7.89$} ($77.89${\scriptsize$\pm2.09$}) & $28.67${\scriptsize$\pm7.64$} ($65.80${\scriptsize$\pm1.91$}) \\
Random             & x & $62.15${\scriptsize$\pm1.84$} ($96.69${\scriptsize$\pm0.12$}) & $1.82${\scriptsize$\pm2.96$} ($24.76${\scriptsize$\pm1.15$}) & $65.97${\scriptsize$\pm1.84$} ($81.77${\scriptsize$\pm0.90$}) & $51.07${\scriptsize$\pm3.49$} ($76.07${\scriptsize$\pm1.67$}) \\
RGbal              & \checkmark & $\boldsymbol{92.67}${\scriptsize$\pm0.58$} ($93.68${\scriptsize$\pm0.41$}) & $\underline{71.68}${\scriptsize$\pm13.83$} ($87.59${\scriptsize$\pm2.24$}) & $70.14${\scriptsize$\pm2.78$} ($81.13${\scriptsize$\pm0.44$}) & $\underline{84.00}${\scriptsize$\pm1.20$} ($86.57${\scriptsize$\pm0.06$}) \\
TCSL-CS           & x & $\underline{91.91}${\scriptsize$\pm0.35$} ($92.83${\scriptsize$\pm0.55$}) & $\boldsymbol{83.37}${\scriptsize$\pm1.33$} ($91.76${\scriptsize$\pm0.44$}) & $\boldsymbol{79.40}${\scriptsize$\pm2.23$} ($84.95${\scriptsize$\pm1.52$}) & $\boldsymbol{84.27}${\scriptsize$\pm1.67$} ($86.93${\scriptsize$\pm0.78$}) \\
\bottomrule
\end{tabular}
}
\caption{We compare the WGA and AVG of different coreset selection methods at $r\!=\!0.1$ across datasets. Results are averaged over $3$ seeds. We use CB ERM for retraining on all of the identified coresets. The best WGA for each dataset is shown in \textbf{bold} and the second best value is \underline{underlined}.
}
\label{tab:tab_2}
\end{table*}

\subsection{Ablation Study}

We conduct an ablation study on the Waterbirds dataset, as it is synthetically constructed and allows for an easy separation of core and spurious features. Here, we evaluate the accuracy of the TCSL score in disentangling the computation of spurious and core feature difficulties. Additional ablation studies on $\text{TCSL}_s$ and $\text{TCSL}_c$, as well as cross-architecture evaluations and robustness of the hyperparameters of $\text{TCSL-CS}$, are provided in Appendix~\ref{sec:app_abb_study}.

In Figure~\ref{fig:ablation_study}, we compute the cosine similarity of the CSL scores under different training schemes. $\text{CSL}(x)$ denotes the class-balanced ERM score on the full dataset, $\text{CSL}(x^c)$ and $\text{CSL}(x^s)$ are scores from ERM models trained on core-only and spurious-only features, respectively. Figure~\ref{fig:ablation_study} demonstrates that $\text{TCSL}_s$ and $\text{TCSL}_c$ achieve high cosine similarity with $\text{CSL}(x^s)$ and $\text{CSL}(x^c)$, respectively. These results confirm that TCSL successfully disentangles the learning processes of core and spurious features. We further evaluate the similarity for the Top 1{,}000 and Top 100 samples with the highest $\text{TCSL}_c$ scores. For the Top 100 samples, $\text{CSL}(x^c)$ and $\text{TCSL}_c$ are nearly identical, with a cosine similarity of 0.918. This further demonstrates that $\text{TCSL}_c$ accurately captures core feature difficulty, effectively identifying samples with the most challenging core features, regardless of the spurious attribute. We additionally include density visualizations of EL2N and TCSL scores across different groups in Appendix~\ref{sec:app_abb_study}.

\begin{figure}[t]
    \centering
    \includegraphics[width=0.8\linewidth]{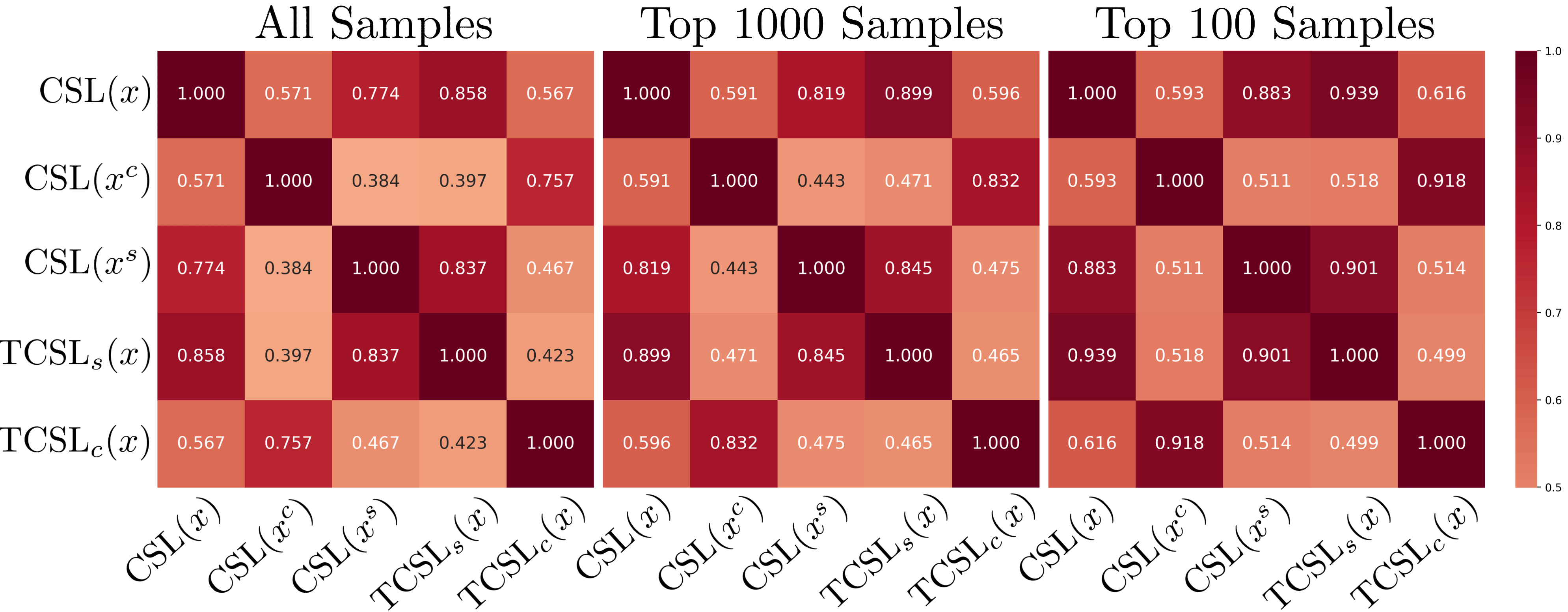}
    \caption{Cosine similarity between the TCSL score components and the CSL scores computed on the Waterbirds dataset for ERM models trained with and without the core and spurious parts of images. Samples with the highest core feature difficulty, as identified by the $\text{TCSL}_c$ score, are specifically evaluated.}
    \label{fig:ablation_study}
\end{figure}

%% file: sec/conclusion.tex
\section{Conclusion}\label{sec:conc}
We present TCSL-CS, a debiasing framework based on coreset selection, specifically built for datasets with strong spurious correlations. TCSL-CS is built on top of our proposed TCSL score, which disentangles the sample score computation for the core and spurious features of images and provides difficulty estimates for each part separately. Across a variety of datasets, we show that a standard ERM model trained on the coresets identified by TCSL-CS achieves performance on par with or exceeding state-of-the-art baselines across multiple datasets, even outperforming models that have access to group labels.

We highlight that our debiasing method based on coreset selection has a significant advantage over other debiasing methods: standard ERM trained on only $10\%$ of the data achieves state-of-the-art performance on highly spurious datasets. Hence, the identified coreset is broadly applicable, for example, it can be used as part of larger training pipelines. The expensive training step is performed once, after which the resulting coreset can be reused as a standalone product. Thus, TCSL-CS is significantly more scalable than other debiasing methods.

%% file: appendix/sec_r.tex
\section{Related Work}\label{sec:app_related_work}

\noindent \textbf{Simplicity Bias.} The simplicity bias phenomenon of deep learning models trained on SGD, where the model learns simpler features before more complex ones even when the latter are more predictive, has been both theoretically analyzed and empirically observed across various settings and architectures \citep{SIMP_BIAS_1, SIMP_BIAS_2, SIMP_BIAS_3}. Recent studies have further extended the notion of simplicity bias to learning under spurious correlations, showing that deep models tend to learn spurious features before capturing the more complex core features \citep{SPARE,COMP_MATTERS,PDE,SIMP_BIAS_THEORY}.

\noindent \textbf{Mitigating Spurious Correlations.} A variety of methods have been proposed to mitigate the learning of spurious correlations during training. SPARE \citep{SPARE} leverages the simplicity bias phenomenon by upsampling underrepresented groups that are identified early in training. Feature reweighting based methods \citep{DFR, AFR} finetune the final layer of an ERM model using a group balanced validation set. Two-stage approaches \citep{JTT, LFF} train two models, where the second model is designed to debias the first, typically by upweighting the ``hard'' samples identified by the initial model. Among these, the methods most similar to our sample score computation scheme are the logit-correction approaches \citep{LC, uLA, MAT}, in which the second model is built directly on top of the first model’s predictions, treating them as prior knowledge. The authors of MAT \citep{MAT} further show that their logit-corrected model assigns similar self-influence scores to majority and minority groups, computed using an external influence estimation method called TRAK \citep{TRAK}. These prior methods highlight the importance of two-stage modeling, which directly motivates the computation of our Two-Stage Cumulative Sample Loss (TCSL).

\noindent \textbf{Coreset Selection.} The goal of coreset selection is to identify a subset of samples that accurately represents the entire dataset, thereby reducing training data size while maintaining model performance \citep{CS_LR}. Most coreset selection methods rely on sample scores that quantify the importance or difficulty of individual samples. These scores vary in formulation but are typically functions of the model’s output computed at one or multiple epochs. Commonly used sample scores include EL2N \citep{EL2N}, Memorization \citep{Memorization} and SelfSup \citep{SelfSup}, all of which have proven effective in various datasets and are widely used as baselines in the literature. Based on these scores, researchers have developed numerous coreset selection strategies, though the optimal selection policy remains an open question. Recent studies \citep{SelfSup, HistCS} have shown that the optimal policy depends on the coreset selection ratio: higher ratios favor pruning ``easy'' samples, while lower ratios favor pruning ``hard'' ones. Furthermore, histogram-based selection approaches \citep{HistCS} have gained popularity for their ability to capture the overall data distribution by sampling from both ``easy'' and ``hard'' regions. In addition, recent work has explored combining score-based functions with feature similarity metrics to improve coreset coverage. For example, the state-of-the-art D2 pruning method~\citep{D2} employs a message-passing algorithm to integrate both scoring- and similarity-based importance metrics.

\noindent \textbf{Coreset Selection Meets Spurious Correlations.} A growing body of recent work has explored the intersection between coreset selection and learning under spurious correlations. \citep{MS_DP} investigates the setting where group labels (spurious attributes) are available and experimentally demonstrates that removing a small fraction of samples with spurious correlations, particularly those with complex (hard) core features, significantly improves worst-group accuracy. \citep{DROP} proposes a coreset selection algorithm aimed at achieving consistently high accuracy across all classes and further shows that their method can be extended to the group-level setting when group labels are available. Most recently and most closely related to our work, \citep{CS_LR} conducts a comprehensive study analyzing how commonly used EL2N and SelfSup scores, and the coreset selection strategies built on them, behave under spurious correlations.

To the best of our knowledge, our method is the first coreset selection algorithm specifically designed to achieve both high average accuracy and high worst-group accuracy, without requiring access to group labels. Our dual objectives are (1) mitigating spurious correlations and (2) reducing training data size. Next, we formalize the problem setting and illustrate why commonly used coreset selection algorithms fail to achieve these objectives.

%% file: appendix/sec_a.tex
\section{Theoretical Analysis}\label{sec:app_theo_anal}

In this section, we provide theoretical justifications for the simplicity bias phenomenon by resorting to the Neural Tangent Kernel (NTK) regime \citep{NTK_1} and a series of spiked data covariance models \citep{paul2007asymptotics,wang2024nonlinear}. This analysis provides a rigorous, dynamic foundation for the claims made in the main paper. We provide a brief introduction to NTK and spiked covariance models at the end of this section to make our theoretical analysis more accessible.

Recall from \citep{NTK_1} (here we simplify and adopt the notation to our setup) that the NTK regime emerges in the infinite-width limit of a neural network. The key consequences of this limiting setting are: 

\begin{enumerate}
    \item The network's output function at initialization $h(x; \mathbf{W}_0)$ (which we refer to as the logit) becomes a draw from a Gaussian Process (Proposition 1 in \citep{NTK_1}).
    \item The Neural Tangent Kernel $K(x, x'; \mathbf{W}) := \nabla_\mathbf{W} h(x; \mathbf{W}) \cdot \nabla_\mathbf{W} h(x'; \mathbf{W})$ converges to a deterministic, positive semi-definite kernel $K(x, x')$ that is constant in time (Theorem 1 in \citep{NTK_1}).
    \item The evolution of the logit outputs $h(x_i; \mathbf{W}_t)$ for the $n$ training samples under gradient flow for the empirical loss $\mathcal{L}(\mathbf{W}) = \frac{1}{n} \sum_{i=1}^n \ell(y_i, h(x_i; \mathbf{W}_t))$ is governed by an exact, deterministic, non-linear Ordinary Differential Equation (ODE) in function space (Theorem 2 in \citep{NTK_1}). For a specific logit $h_j(t) \equiv h(x_j; \mathbf{W}_t)$, the dynamic is
\begin{equation}\label{eq:flow_ode}
\frac{\partial h(x_j; \mathbf{W}_t)}{\partial t} = -\frac{1}{n} \sum_{i=1}^n K(x_j, x_i) \frac{\partial \ell(y_i, h(x_i; \mathbf{W}_t))}{\partial h(x_i; \mathbf{W}_t)}
\end{equation}
\end{enumerate}

To analyze the dynamics, we require a set of foundational assumptions. These are standard in the theoretical analysis of deep learning and are necessary to make the problem analytically tractable.

The following is a standard assumption in NTK, as it is a direct consequence of the Central Limit Theorem applied to wide networks initialized with zero-mean symmetric weights (e.g., Gaussian). It allows us to analyze the deterministic expected dynamics rather than a single stochastic trajectory.

\begin{assumption}[Symmetric Initialization]\label{assum:sym_init}
    The initialization distribution $p_\mathbf{W}$ is symmetric such that the resulting Gaussian Process $h(x; \mathbf{W}_0)$ has a zero mean function, i.e., $\mathbb{E}_{\mathbf{W}_0}[h(x; \mathbf{W}_0)] = 0$ for all $x$.
\end{assumption}

The following proposition shows that under symmetric initialization, the expected initial gradient of the logistic loss is non-zero and points exactly in the direction of the negative label with a constant factor of $1/2$.

\begin{proposition}\label{prop:init_grad}
    Suppose Assumption \ref{assum:sym_init} holds. Let $\ell(y, h) = \log(1 + e^{-y h})$ be the logistic loss. Let $g_i(t) = \frac{\partial \ell(y_i, h(x_i; \mathbf{W}_t))}{\partial h(x_i; \mathbf{W}_t)} = -y_i \sigma(-y_i h(x_i; \mathbf{W}_t))$, where $\sigma(z) = (1 + e^{-z})^{-1}$ is the sigmoid function. Then,
\begin{equation}
    \mathbb{E}_{\mathbf{W}_0}[g_i(0)] = -\frac{1}{2} y_i
\end{equation}
\end{proposition}
\begin{proof}
    Let $Z_i = -y_i h(x_i; \mathbf{W}_0)$. By Assumption \ref{assum:sym_init}, $h(x_i; \mathbf{W}_0)$ has a distribution $p(h)$ that is symmetric about zero (i.e., $p(h) = p(-h)$). Thus, the distribution of $Z_i$, $q(z)$, is likewise symmetric about zero. 
Recall that the sigmoid function satisfies the identity $\sigma(-z) = 1 - \sigma(z)$. Thus, by linearity of expectation $\mathbb{E}[\sigma(-Z_i)] = 1 - \mathbb{E}[\sigma(Z_i)]$.
Since the distribution $q(z)$ is symmetric, the random variables $Z_i$ and $-Z_i$ are identically distributed. Thus, $\mathbb{E}[\sigma(Z_i)] = \mathbb{E}[\sigma(-Z_i)]$.
Substituting this into the previous identity yields $\mathbb{E}[\sigma(Z_i)] = 1/2$.
Therefore, 
$\mathbb{E}_{\mathbf{W}_0}[g_i(0)] = -y_i \mathbb{E}_{\mathbf{W}_0}[\sigma(-y_i h(x_i; \mathbf{W}_0))] = -y_i (1/2)$.
\end{proof}

The following assumption is necessary to formally disentangle the learning of core and spurious features. It allows us to analyze their dynamics independently by partitioning the model's parameters into those that process core features and those that process spurious features.
\begin{assumption}[Model Decomposition]\label{assum:decomp}
    The model architecture additively separates core and spurious features, $h(x; \mathbf{W}) = h_c(x^c; \mathbf{W}^c) + h_s(x^s; \mathbf{W}^s)$, where the parameter sets $\mathbf{W}^c$ and $\mathbf{W}^s$ are disjoint. 
\end{assumption}

The following proposition formally proves that an additive decomposition of the logit function, combined with disjoint parameters, directly implies an additive decomposition of the Tangent Kernel.
\begin{proposition}\label{prop:decomp_kernel}
    Suppose Assumption \ref{assum:decomp} holds. Then, the NTK is additively decomposable, i.e., 
    \begin{equation}
        K(x, x') = K_c(x^c, x'^c) + K_s(x^s, x'^s)
    \end{equation}
     where 
     \begin{equation}
         K_c(x, x') = \nabla_{\mathbf{W}^c} h_c(x^c; \mathbf{W}^c) \cdot \nabla_{\mathbf{W}^c} h_c(x'^c; \mathbf{W}^c)
     \end{equation}
     is the NTK of the core subnetwork, computed with respect to its own parameters $\mathbf{W}^c$,
and  
\begin{equation}
    K_s(x, x') = \nabla_{\mathbf{W}^s} h_s(x^s; \mathbf{W}^s) \cdot \nabla_{\mathbf{W}^s} h_s(x'^s; \mathbf{W}^s)
\end{equation}
is the NTK of the spurious subnetwork, computed with respect to its own parameters $\mathbf{W}^s$.
\end{proposition}
\begin{proof}
    By definition, $K(x, x'; \mathbf{W}) = \nabla_\mathbf{W} h(x; \mathbf{W}) \cdot \nabla_\mathbf{W} h(x'; \mathbf{W})$. The total gradient $\nabla_\mathbf{W} h(x; \mathbf{W})$ is the concatenation $[\nabla_{\mathbf{W}^c} h(x; \mathbf{W}) \, ; \, \nabla_{\mathbf{W}^s} h(x; \mathbf{W})]$.
    By Assumption \ref{assum:decomp}, the sub-functions are functionally independent of the other's parameters: $\nabla_{\mathbf{W}^s} h_c(x^c; \mathbf{W}^c) = \mathbf{0}$ and $\nabla_{\mathbf{W}^c} h_s(x^s; \mathbf{W}^s) = \mathbf{0}$.
    Thus, $\nabla_{\mathbf{W}^c} h(x; \mathbf{W}) = \nabla_{\mathbf{W}^c} (h_c + h_s) = \nabla_{\mathbf{W}^c} h_c(x^c; \mathbf{W}^c)$.
    Similarly, $\nabla_{\mathbf{W}^s} h(x; \mathbf{W}) = \nabla_{\mathbf{W}^s} h_s(x^s; \mathbf{W}^s)$.
    The total gradient is $\nabla_\mathbf{W} h(x; \mathbf{W}) = [\nabla_{\mathbf{W}^c} h_c \, ; \, \nabla_{\mathbf{W}^s} h_s]$.
    Let $g_c(x) = \nabla_{\mathbf{W}^c} h_c(x)$ and $g_s(x) = \nabla_{\mathbf{W}^s} h_s(x)$. The inner product is therefore
    \begin{align*}
    K(x,x') 
      &= [g_c(x); g_s(x)] \cdot [g_c(x'); g_s(x')] \\
      &= g_c(x)\cdot g_c(x') + g_s(x)\cdot g_s(x') \\
      &= K_c(x,x') + K_s(x,x'),
    \end{align*}
    which completes the proof.
\end{proof}

The following assumption formally defines the setting of spurious correlation, where the dataset is imbalanced such that the spurious attribute $a_i$ is predictive of the true label $y_i$ for a fraction $\alpha>1/2$ of the data.
\begin{assumption}[Data Structure]\label{assum:data_structure}
    The dataset $\mathcal{D}$ of size $n$ is partitioned into $G_1 = \{i : y_i = a_i\}$ of size $n_1 = \alpha n$ and $G_2 = \{i : y_i = -a_i\}$ of size $n_2 = (1-\alpha) n$, with $\alpha \in (1/2, 1]$.
\end{assumption}

\subsection{Homogeneous Spiked Model}\label{sec:homo-spiked}

We begin with a simple model that assumes all samples have uniform feature strength. The strength of this setting is in its simple statements, which we find valuable for pedagogical reasons.

The following assumption is inspired by the Spiked Covariance Model from high-dimensional statistics. It provides an idealized and analytically simple model to isolate the competition between the core signal (strength $\beta^c$) and the spurious signal (strength $\beta^s$).
\begin{assumption}[Homogeneous Spiked Model]\label{assum:homo_spike}
    The NTK components are perfectly aligned with the latent data structure and have a uniform, rank-1 spiked structure:
    \begin{equation}
        K_c(x_j^c, x_i^c) = \beta^c y_j y_i, \qquad \beta^c > 0
    \end{equation}
and 
\begin{equation}
    K_s(x_j^s, x_i^s) = \beta^s a_j a_i, \qquad \beta^s > 0
\end{equation}
\end{assumption}

The following theorem provides the exact initial velocities of the expected logits. It shows that the core subnetwork always learns in the direction of the true label, while the spurious subnetwork learns in the direction of the spurious attribute, with a velocity amplified by the data imbalance $2\alpha-1$.

\begin{theorem}[Initial Velocity in Homogeneous Spiked Model]\label{thm:homo_velocity}
    Define the expected logit $\bar{h}_t(x) = \mathbb{E}_{\mathbf{W}_0}[h(x; \mathbf{W}_t)]$ and let Assumptions \ref{assum:sym_init}-\ref{assum:homo_spike} hold. Then, the initial velocities of the expected subnetworks are
    \begin{equation}
        \frac{\partial \bar{h}_t^c(x_j)}{\partial t} \Big|_{t=0} = \frac{\beta^c}{2} y_j,
    \end{equation}
    and
\begin{equation}
    \frac{\partial \bar{h}_t^s(x_j)}{\partial t} \Big|_{t=0} = \frac{\beta^s (2\alpha - 1)}{2} a_j
\end{equation}
where $\bar{h}_t^c(x) = \mathbb{E}_{\mathbf{W}_0}[h_c(x; \mathbf{W}_t)]$ and $\bar{h}_t^s(x) = \mathbb{E}_{\mathbf{W}_0}[h_s(x; \mathbf{W}_t)]$.
\end{theorem}
\begin{proof}
Let $\bar{h}_{t,j} = \bar{h}_t(x_j)$. The dynamics of $\bar{h}_{t,j}$ are $\frac{\partial \bar{h}_{t,j}}{\partial t} = \mathbb{E}_{\mathbf{W}_0} \left[ \frac{\partial h(x_j; \mathbf{W}_t)}{\partial t} \right]$. At $t=0$
\begin{equation}
   \frac{\partial \bar{h}_{t,j}}{\partial t} \Big|_{t=0} = \mathbb{E}_{\mathbf{W}_0} \left[ -\frac{1}{n} \sum_{i=1}^n K(x_j, x_i) g_i(0) \right] 
\end{equation}
By linearity of expectation and the fact that $K$ is deterministic
\begin{equation}
    \frac{\partial \bar{h}_{t,j}}{\partial t} \Big|_{t=0} = -\frac{1}{n} \sum_{i=1}^n K(x_j, x_i) \mathbb{E}_{\mathbf{W}_0}[g_i(0)]
\end{equation}
Leveraging the result of Proposition \ref{prop:init_grad}
\begin{equation}
\begin{aligned}
\left.\frac{\partial \bar{h}_{t,j}}{\partial t}\right|_{t=0}
&= -\frac{1}{n} \sum_{i=1}^n K(x_j, x_i)\left(-\frac{y_i}{2}\right) \\
&= \frac{1}{2n} \sum_{i=1}^n K(x_j, x_i) y_i.
\end{aligned}
\end{equation}
Using Proposition \ref{prop:decomp_kernel}, we analyze the subnetworks. For the core subnetwork
\begin{equation}
\begin{aligned}
\frac{\partial \bar{h}_t^c(x_j)}{\partial t} \Big|_{t=0}
    &= \frac{1}{2n} \sum_{i=1}^n K_c(x_j^c, x_i^c) y_i \\
    &= \frac{1}{2n} \sum_{i=1}^n (\beta^c y_j y_i) y_i \\
    &= \frac{\beta^c y_j}{2n} \sum_{i=1}^n y_i^2 \\
    &= \frac{\beta^c}{2} y_j,
\end{aligned}
\end{equation}
where we used Assumption \ref{assum:homo_spike} and $y_i^2 = 1$. Similarly, for the spurious subnetwork
\begin{equation}
\begin{aligned}
\frac{\partial \bar{h}_t^s(x_j)}{\partial t} \Big|_{t=0}
    &= \frac{1}{2n} \sum_{i=1}^n K_s(x_j^s, x_i^s) y_i \\
    &= \frac{1}{2n} \sum_{i=1}^n (\beta^s a_j a_i) y_i \\
    &= \frac{\beta^s a_j}{2n} \sum_{i=1}^n a_i y_i .
\end{aligned}
\end{equation}
By Assumption \ref{assum:data_structure}, $\sum_{i=1}^n a_i y_i = n_1 - n_2 = n(2\alpha - 1)$. This yields $\frac{\beta^s a_j}{2n} (n(2\alpha-1)) = \frac{\beta^s (2\alpha - 1)}{2} a_j$.
\end{proof}

The following corollary provides the precise condition for simplicity bias in the homogeneous model. It is a competition between the core feature strength $\beta^c$ and the spurious feature strength $\beta^s$ modulated by the data imbalance $2\alpha-1$.
\begin{corollary}\label{cor:homo_bias_condition}
    Let the expected initial growth rate of the true core margin be $R_c := \frac{\partial}{\partial t} \mathbb{E}_{\mathbf{W}_0}[y_j h_c(x_j; \mathbf{W}_t)] \Big|_{t=0}$.
Let the expected initial growth rate of the true spurious margin be $R_s := \frac{\partial}{\partial t} \mathbb{E}_{\mathbf{W}_0}[a_j h_s(x_j; \mathbf{W}_t)] \Big|_{t=0}$.
Then $R_c = \beta^c / 2$ and $R_s = \frac{\beta^s (2\alpha - 1)}{2}$.
The model exhibits simplicity bias ($R_s > R_c$) if and only if $\beta^s (2\alpha - 1) > \beta^c$.
\end{corollary}
The following theorem corresponds to Theorem 1 in the main paper. We demonstrate the direct consequence of the velocity imbalance. At the start of training, the expected loss on the majority group immediately decreases, while the expected loss on the minority group immediately increases, demonstrating the simplicity bias.
\begin{theorem}[Initial Loss Divergence in Homogeneous Spiked Model]\label{thm:homo_loss_divergence}
    Let the simplicity bias condition $\beta^s (2\alpha - 1) > \beta^c$ from Corollary \ref{cor:homo_bias_condition} hold. Let $\bar{m}_t(x_j) = \mathbb{E}_{\mathbf{W}_0}[y_j h(x_j; \mathbf{W}_t)]$ be the expected margin.
Then, there exists a time $T > 0$ such that for all $t \in (0, T)$:
\begin{enumerate}
    \item For $j \in G_1$ (majority group, $y_j = a_j$), the expected margin $\bar{m}_t(x_j)$ is positive, and the loss $\ell(\bar{m}_t(x_j))$ is less than $\log(2)$.
    \item For $j \in G_2$ (minority group, $y_j = -a_j$), the expected margin $\bar{m}_t(x_j)$ is negative, and the loss $\ell(\bar{m}_t(x_j))$ is greater than $\log(2)$.
\end{enumerate}
\end{theorem}
\begin{proof}
By Assumption \ref{assum:sym_init}, the initial expected logit is $\bar{h}_0(x_j) = 0$, so the initial expected margin is $\bar{m}_0(x_j) = 0$. We compute the initial time-derivative of the expected margin
$$\frac{\partial \bar{m}_t(x_j)}{\partial t} \Big|_{t=0} = y_j \left( \frac{\partial \bar{h}_t^c(x_j)}{\partial t} \Big|_{t=0} + \frac{\partial \bar{h}_t^s(x_j)}{\partial t} \Big|_{t=0} \right)$$
Using Theorem \ref{thm:homo_velocity} and $y_j^2=1$
\begin{equation}
\begin{aligned}
\frac{\partial \bar{m}_t(x_j)}{\partial t} \Big|_{t=0}
    &= y_j \left( \frac{\beta^c}{2} y_j + \frac{\beta^s (2\alpha - 1)}{2} a_j \right) \\
    &= \frac{\beta^c}{2} y_j^2 + \frac{\beta^s (2\alpha - 1)}{2} (y_j a_j) \\
    &= \frac{\beta^c}{2} + \frac{\beta^s (2\alpha - 1)}{2} (y_j a_j).
\end{aligned}
\end{equation}
Note that the loss function $\ell(\bar{m}) = \log(1 + e^{-\bar{m}})$ is strictly monotonically decreasing in $\bar{m}$.
\begin{enumerate}
    \item For $j \in G_1$, $y_j a_j = 1$. The initial velocity is $R_c + R_s = \frac{\beta^c}{2} + \frac{\beta^s(2\alpha-1)}{2} > 0$. Since $\bar{m}_0(x_j) = 0$ and $\frac{\partial \bar{m}_t(x_j)}{\partial t} \Big|_{t=0} > 0$, there exists $T_1 > 0$ such that $\bar{m}_t(x_j) > 0$ for $t \in (0, T_1)$. Thus, $\ell(\bar{m}_t(x_j)) < \ell(0) = \log(2)$.
    \item For $j \in G_2$, $y_j a_j = -1$. The initial velocity is $R_c - R_s = \frac{\beta^c}{2} - \frac{\beta^s(2\alpha-1)}{2} < 0$ by the simplicity bias condition. Since $\bar{m}_0(x_j) = 0$ and $\frac{\partial \bar{m}_t(x_j)}{\partial t} \Big|_{t=0} < 0$, there exists $T_2 > 0$ such that $\bar{m}_t(x_j) < 0$ for $t \in (0, T_2)$. Thus, $\ell(\bar{m}_t(x_j)) > \ell(0) = \log(2)$.
\end{enumerate}
Let $T = \min(T_1, T_2)$. For $t \in (0, T)$, both statements hold.
\end{proof}

The following theorem characterizes the initial curvature of the learning path. It shows how the model begins to decelerate, or saturate, as a function of the feature strengths. We need to define some notation and make an approximation.

Let $\mathbf{h}_t = [h(x_1; \mathbf{W}_t), ..., h(x_n; \mathbf{W}_t)]^\top$ be the vector of logits, and let $\mathbf{g}_t = [g_1(t), ..., g_n(t)]^\top$. The flow ODE in vector form is $\dot{\mathbf{h}}_t = -\frac{1}{n} K \mathbf{g}_t$. Let $\bar{\mathbf{h}}_t = \mathbb{E}[\mathbf{h}_t]$.
    The expected dynamics are $\dot{\bar{\mathbf{h}}}_t = -\frac{1}{n} K \mathbb{E}[\mathbf{g}_t]$.
    From Theorem \ref{thm:homo_velocity}, we have $\dot{\bar{\mathbf{h}}}_0 = \frac{1}{2n} K \mathbf{y}$.
    The expected acceleration is $\ddot{\bar{\mathbf{h}}}_t = \frac{\partial}{\partial t} \dot{\bar{\mathbf{h}}}_t = -\frac{1}{n} K \mathbb{E}[\dot{\mathbf{g}}_t]$.
    Element-wise, $\dot{g}_i(t) = \frac{\partial g_i(t)}{\partial h_i(t)} \dot{h}_i(t) = \sigma'(-y_i h_i(t)) \dot{h}_i(t)$.
    We will approximate by linearization $\mathbb{E}[\sigma'(-y_i h_i(0))] \approx \sigma'(-y_i \bar{h}_i(0)) = \sigma'(0) = 1/4$, leading to $\mathbb{E}[\dot{\mathbf{g}}_0] \approx \frac{1}{4} \dot{\bar{\mathbf{h}}}_0$.
    
\begin{theorem}[Initial Acceleration in Homogeneous Spiked Model]\label{thm:homo_acceleration}
    Let $\bar{a}_{0,j} := \frac{\partial^2 \bar{h}_t(x_j)}{\partial t^2} \Big|_{t=0}$ denote the initial acceleration of the expected logit for sample $j$. Under Assumptions \ref{assum:sym_init}-\ref{assum:homo_spike}, the total acceleration is
    \begin{align}
    \bar{a}_{0,j} = -\frac{1}{8} \Big[ \big( (\beta^c)^2 + \beta^c \beta^s (2\alpha - 1)^2 \big) y_j + (2\alpha - 1) \big( \beta^c \beta^s + (\beta^s)^2 \big) a_j \Big]
    \end{align}
    The subnetwork accelerations are
    \begin{align}
        \bar{a}_{0,j}^c &= -\frac{1}{8} \left[ (\beta^c)^2 + \beta^c \beta^s (2\alpha - 1)^2 \right] y_j \\
        \bar{a}_{0,j}^s &= -\frac{1}{8} (2\alpha - 1) \left[ \beta^c \beta^s + (\beta^s)^2 \right] a_j
    \end{align}
\end{theorem}
\begin{proof}
    
    Adopting the approximation described in the preceding discussion,
    \begin{equation}
        \ddot{\bar{\mathbf{h}}}_0 \approx -\frac{1}{4n} K \dot{\bar{\mathbf{h}}}_0
    \end{equation}
    Substituting $\dot{\bar{\mathbf{h}}}_0 = \frac{1}{2n} K \mathbf{y}$ we have
    \begin{equation}
        \ddot{\bar{\mathbf{h}}}_0 \approx -\frac{1}{8n^2} K^2 \mathbf{y}
    \end{equation}
    We compute $K^2 \mathbf{y}$ using $K = \beta^c \mathbf{y}\mathbf{y}^\top + \beta^s \mathbf{a}\mathbf{a}^\top$ and the inner products $\mathbf{y}^\top\mathbf{y}=n$, $\mathbf{a}^\top\mathbf{a}=n$, and $\mathbf{y}^\top\mathbf{a}=n(2\alpha-1)$. We have
    $K\mathbf{y} = \beta^c \mathbf{y}(\mathbf{y}^\top \mathbf{y}) + \beta^s \mathbf{a}(\mathbf{a}^\top \mathbf{y}) = n\beta^c \mathbf{y} + n\beta^s(2\alpha-1)\mathbf{a}$. Also,
    $K\mathbf{a} = \beta^c \mathbf{y}(\mathbf{y}^\top \mathbf{a}) + \beta^s \mathbf{a}(\mathbf{a}^\top \mathbf{a}) = n\beta^c(2\alpha-1)\mathbf{y} + n\beta^s\mathbf{a}$.
    \begin{align*}
    K^2\mathbf{y}
    &= K\bigl(n\beta^c \mathbf{y} + n\beta^s(2\alpha-1)\mathbf{a}\bigr) \\
    &= n\beta^c(K\mathbf{y}) + n\beta^s(2\alpha-1)(K\mathbf{a}) \\
    &= n\beta^c \bigl[n\beta^c \mathbf{y} + n\beta^s(2\alpha-1)\mathbf{a}\bigr] + n\beta^s(2\alpha-1)\bigl[n\beta^c(2\alpha-1)\mathbf{y} + n\beta^s\mathbf{a}\bigr] \\
    &= \bigl[n^2(\beta^c)^2 + n^2\beta^c\beta^s(2\alpha-1)^2\bigr]\mathbf{y} + \bigl[n^2\beta^c\beta^s(2\alpha-1) + n^2(\beta^s)^2(2\alpha-1)\bigr]\mathbf{a}.
    \end{align*}
    Dividing by $-8n^2$ gives the total acceleration $\bar{a}_{0,j}$. The subnetwork accelerations $\bar{a}_{0,j}^c = -\frac{1}{4n} (K_c \dot{\bar{\mathbf{h}}}_0)_j$ and $\bar{a}_{0,j}^s = -\frac{1}{4n} (K_s \dot{\bar{\mathbf{h}}}_0)_j$ follow by isolating the respective kernel terms.
\end{proof}

The following corollary breaks down the acceleration by group. The majority group, with aligned signals, always decelerates. The minority group, with conflicting signals, can accelerate if the spurious signal is strong. When both decelerate, the majority group decelerates faster.
\begin{corollary}[Group-Specific Acceleration Dynamics]\label{cor:homo_accel_dynamics}
    Let $A_{maj}$ and $A_{min}$ be the scalar acceleration of the margin ($A_j = y_j \bar{a}_{0,j}$) for the majority ($j \in G_1$) and minority ($j \in G_2$) groups, respectively.
    \begin{enumerate}
        \item \textbf{Universal Majority Deceleration:} The majority group always decelerates (saturates):
        \begin{align}
        A_{maj}
        &= -\frac{1}{8} \Big[
            (\beta^c)^2
            + \beta^c \beta^s (2\alpha - 1)^2 \notag
            + (2\alpha - 1)\big( \beta^c \beta^s + (\beta^s)^2 \big)
        \Big] < 0
        \end{align}
        \item \textbf{Minority Acceleration Regime:} The minority group accelerates ($A_{min} > 0$) if the spurious feature is sufficiently strong:
        \begin{equation}
            (2\alpha - 1) (\beta^s)^2 + (2\alpha - 1)\beta^c\beta^s > (\beta^c)^2 + \beta^c\beta^s(2\alpha - 1)^2
        \end{equation}
        \item \textbf{Relative Saturation Strength:} In the regime where $A_{min} < 0$ (both groups decelerate), the majority group decelerates with a strictly greater magnitude:
        \begin{equation}
            |A_{maj}| > |A_{min}|
        \end{equation}
    \end{enumerate}
\end{corollary}

\subsection{Heterogeneous Spiked Model}\label{sec:het-spiked}

We now relax the assumption of uniform feature strength, allowing each sample to have its own strength. This model is more realistic as it allows for ``hard'' (low $\beta$) and ``easy'' (high $\beta$) samples within each group, but retains a rank-1 structure that permits exact closed-form solutions.

\begin{assumption}[Heterogeneous Spiked Model]\label{assum:het_spike}
    The NTK components decompose into rank-1 matrices based on sample-specific strengths:
    \begin{equation}
    \begin{aligned}
    K_c(x_j, x_i) &= (\beta^c_i \beta^c_j) y_i y_j \\
    K_s(x_j, x_i) &= (\beta^s_i \beta^s_j) a_i a_j,
    \end{aligned}
    \end{equation}
    where $\beta^c_i > 0$ and $\beta^s_i > 0$ are the core and spurious strengths of sample $i$. We assume these are bounded, $b^c \le \beta^c_i \le B^c$ and $b^s \le \beta^s_i \le B^s$.
\end{assumption}
\begin{definition}[Global Dataset Statistics]\label{def:global_stats}
    We define the following scalar summaries of the dataset's feature structure:
    \begin{align}
        S_{cc} &= \sum_{i=1}^n \beta^c_i &  \\ 
        S_{sc} &= \sum_{i=1}^n \beta^s_i a_i y_i = \sum_{i \in G_1} \beta^s_i - \sum_{i \in G_2} \beta^s_i &  \\
        Q_{cc} &= \sum_{i=1}^n (\beta^c_i)^2  \\
        Q_{ss} &= \sum_{i=1}^n (\beta^s_i)^2 \\
        Q_{cs} &= \sum_{i=1}^n \beta^c_i \beta^s_i a_i y_i = \sum_{i \in G_1} \beta^c_i \beta^s_i - \sum_{i \in G_2} \beta^c_i \beta^s_i
    \end{align}
    Further, we define the global energy constants
    \begin{equation}
    \begin{aligned}
    C_{core} &= S_{cc} Q_{cc} + S_{sc} Q_{cs}, \\
    C_{spur} &= S_{cc} Q_{cs} + S_{sc} Q_{ss}.
    \end{aligned}
    \end{equation}
\end{definition}
These scalar quantities summarize the aggregate ``global'' properties of the dataset under the heterogeneous Spiked model. Note that $S_{sc}$ and $Q_{cs}$ are positive by Assumption \ref{assum:data_structure}.

The following theorem shows that the initial velocity for a sample is a product of its {local} feature strength ($\beta_k$) and the {global} alignment of the entire dataset ($S_{cc}$ or $S_{sc}$).

\begin{theorem}[Initial Velocities in Heterogeneous Spiked Model]\label{thm:global_velocity}
    Let $\bar{h}_t(x_k)$ be the expected logit for sample $k$. Under Assumptions \ref{assum:sym_init}-\ref{assum:data_structure} and \ref{assum:het_spike}, the initial velocities of the subnetwork logits are
    \begin{equation}
    \begin{aligned}
    \left.\frac{\partial \bar{h}_t^c(x_k)}{\partial t}\right|_{t=0}
    &= \frac{S_{cc}}{2n} \beta^c_k y_k, \\
    \left.\frac{\partial \bar{h}_t^s(x_k)}{\partial t}\right|_{t=0}
    &= \frac{S_{sc}}{2n} \beta^s_k a_k.
    \end{aligned}
    \end{equation}
\end{theorem}
\begin{proof}
    Let $\mathbf{u}_c \in \mathbb{R}^n$ have entries $(\mathbf{u}_c)_i = \beta^c_i y_i$ and $\mathbf{u}_s \in \mathbb{R}^n$ have entries $(\mathbf{u}_s)_i = \beta^s_i a_i$. Then $K_c = \mathbf{u}_c \mathbf{u}_c^\top$ and $K_s = \mathbf{u}_s \mathbf{u}_s^\top$.
    From the proof of Theorem \ref{thm:homo_velocity}, $\dot{\bar{\mathbf{h}}}_0 = \frac{1}{2n} K \mathbf{y}$.
    $\dot{\bar{\mathbf{h}}}^c_0 = \frac{1}{2n} (\mathbf{u}_c \mathbf{u}_c^\top) \mathbf{y} = \frac{1}{2n} \mathbf{u}_c (\mathbf{u}_c^\top \mathbf{y}) = \frac{1}{2n} \mathbf{u}_c (\sum_i \beta^c_i y_i^2) = \frac{S_{cc}}{2n} \mathbf{u}_c$.
    $\dot{\bar{\mathbf{h}}}^s_0 = \frac{1}{2n} (\mathbf{u}_s \mathbf{u}_s^\top) \mathbf{y} = \frac{1}{2n} \mathbf{u}_s (\mathbf{u}_s^\top \mathbf{y}) = \frac{1}{2n} \mathbf{u}_s (\sum_i \beta^s_i a_i y_i) = \frac{S_{sc}}{2n} \mathbf{u}_s$.
    The $k$-th entry of $\dot{\bar{\mathbf{h}}}^c_0$ is $\frac{S_{cc}}{2n} \beta^c_k y_k$, and for $\dot{\bar{\mathbf{h}}}^s_0$ it is $\frac{S_{sc}}{2n} \beta^s_k a_k$.
\end{proof}

The following corollary refines the simplicity bias condition. It is now a sample-specific condition that depends on the ratio of the sample's local strengths ($\beta^s_k / \beta^c_k$) versus the ratio of the dataset's global alignments ($S_{cc} / S_{sc}$). Interestingly, this explains why some minority samples might be learned while others (with high $\beta^s_k$) are not.
\begin{corollary}[Sample-Wise Simplicity Bias]\label{cor:global_bias_condition}
    The margin velocity for sample $k$ is $\dot{\bar{m}}_k(0) = \frac{1}{2n} (S_{cc} \beta^c_k + S_{sc} \beta^s_k (y_k a_k))$.
    For a minority sample ($k \in G_2$, $y_k a_k = -1$), the loss increases at initialization if:
    \begin{equation}
        \beta^s_k S_{sc} > \beta^c_k S_{cc}
    \end{equation}
\end{corollary}

\begin{theorem}[Initial Acceleration in Heterogeneous Spiked Model]\label{thm:global_acceleration}
    Recall the global energy constants from Definition \ref{def:global_stats}
    \begin{equation}
    \begin{aligned}
    C_{core} &= S_{cc} Q_{cc} + S_{sc} Q_{cs}, \\
    C_{spur} &= S_{cc} Q_{cs} + S_{sc} Q_{ss}.
    \end{aligned}
    \end{equation}
    The initial acceleration of the expected logit for sample $k$ is
    \begin{equation}
        \bar{a}_{0,k} = -\frac{1}{8n^2} \left[ \beta^c_k C_{core} y_k + \beta^s_k C_{spur} a_k \right]
    \end{equation}
\end{theorem}
\begin{proof}
    We start with the approximation $\ddot{\bar{\mathbf{h}}}_0 \approx -\frac{1}{4n} K \dot{\bar{\mathbf{h}}}_0$. From Theorem \ref{thm:global_velocity}, $\dot{\bar{\mathbf{h}}}_0 = \frac{1}{2n} (S_{cc} \mathbf{u}_c + S_{sc} \mathbf{u}_s)$.
    \begin{align*}
    \ddot{\bar{\mathbf{h}}}_0
    &\approx -\frac{1}{8n^2} (\mathbf{u}_c \mathbf{u}_c^\top + \mathbf{u}_s \mathbf{u}_s^\top)
        (S_{cc} \mathbf{u}_c + S_{sc} \mathbf{u}_s) \\
    &= -\frac{1}{8n^2} \Big[
        \mathbf{u}_c \big( S_{cc} (\mathbf{u}_c^\top \mathbf{u}_c)
                         + S_{sc} (\mathbf{u}_c^\top \mathbf{u}_s) \big)
      + \mathbf{u}_s \big( S_{cc} (\mathbf{u}_s^\top \mathbf{u}_c)
                         + S_{sc} (\mathbf{u}_s^\top \mathbf{u}_s) \big)
    \Big].
    \end{align*}
    Using the definition of inner product and the quantities defined in Definition \ref{def:global_stats} and the proof of Theorem \ref{thm:global_velocity} we have $\mathbf{u}_c^\top \mathbf{u}_c = Q_{cc}$, $\mathbf{u}_s^\top \mathbf{u}_s = Q_{ss}$, and $\mathbf{u}_c^\top \mathbf{u}_s = \mathbf{u}_s^\top \mathbf{u}_c = Q_{cs}$ leading to
    \begin{align}
    \ddot{\bar{\mathbf{h}}}_0
    &\approx -\frac{1}{8n^2} \Big[
        \mathbf{u}_c (S_{cc} Q_{cc} + S_{sc} Q_{cs})
        + \mathbf{u}_s (S_{cc} Q_{cs} + S_{sc} Q_{ss})
    \Big].
    \end{align}
    Substituting the definitions of $C_{core}$, $C_{spur}$, $\mathbf{u}_c$, and $\mathbf{u}_s$ yields the element-wise result for $\bar{a}_{0,k}$ stated in the Theorem.
\end{proof}
Note that the initial acceleration is also a product of local sample strengths ($\beta_k$) and the global energy constants ($C_{core}, C_{spur}$) that depend on the aggregate statistics of the dataset.

The following corollary demonstrates that the qualitative dynamics hold. Majority samples always decelerate. A minority sample's acceleration depends on its specific $\beta^s_k$ vs $\beta^c_k$ profile, explaining intra-group variance. When all samples decelerate, majority samples do so more rapidly.
\begin{corollary}[Group-Specific Acceleration in Heterogeneous Spiked Model]\label{cor:global_accel_dynamics}
    Let $A_k = y_k \bar{a}_{0,k}$ be the margin acceleration for sample $k$.
    \begin{enumerate}
        \item \textbf{Universal Majority Deceleration:} For any majority sample $k \in G_1$ ($y_k = a_k$):
        \begin{equation}
            A_k = -\frac{1}{8n^2} (\beta^c_k C_{core} + \beta^s_k C_{spur}) < 0
        \end{equation}
        (Assuming $S_{sc}, Q_{cs} > 0$).
        \item \textbf{Minority Acceleration Regime:} For a minority sample $k \in G_2$ ($y_k = -a_k$), the acceleration is:
        \begin{equation}
            A_k = -\frac{1}{8n^2} (\beta^c_k C_{core} - \beta^s_k C_{spur})
        \end{equation}
        The sample accelerates ($A_k > 0$) if its spurious-weighted energy contribution exceeds its core-weighted one:
        \begin{equation}
            \beta^s_k C_{spur} > \beta^c_k C_{core}
        \end{equation}
        
        \item \textbf{Relative Saturation Strength:} In the regime where a minority sample also decelerates ($A_k < 0$, i.e., $\beta^c_k C_{core} > \beta^s_k C_{spur}$), its deceleration is strictly weaker than that of a majority sample with the same local strengths $\beta^c_k, \beta^s_k$:
        \begin{equation}
            |A_{k \in G_1}| - |A_{k \in G_2}| = \frac{1}{4n^2} \beta^s_k C_{spur} > 0
        \end{equation}
    \end{enumerate}
\end{corollary}

%%%% Edge cases subsection for the theory.

\subsection{Validation of Our Theoretical Setting and Assumptions}

\subsubsection{Analysis on the Linearization in Theorem \ref{thm:homo_acceleration}}

We establish guarantees for the linearization approximation $\mathbb{E}_{\mathbf{W}_0}[\sigma'(-y_i h(x_i; \mathbf{W}_0))] \approx \sigma'(0) = 1/4$ utilized in the derivation of the initial expected acceleration. In the infinite-width NTK limit under symmetric initialization, the network output $h(x_i; \mathbf{W}_0)$ converges in distribution to a centered Gaussian process. Consequently, the random variable $Z_i = -y_i h(x_i; \mathbf{W}_0)$ follows a Gaussian distribution $\mathcal{N}(0, \sigma_h^2)$, where $\sigma_h^2 = K(x_i, x_i)$ denotes the initial function space variance.

\begin{proposition}
Let $Z \sim \mathcal{N}(0, \sigma_h^2)$. The absolute error induced by the linearization approximation evaluated at initialization is bounded by
\begin{equation}
0 \le \frac{1}{4} - \mathbb{E}[\sigma'(Z)] \le \frac{\sigma_h^2}{16}
\end{equation}
\end{proposition}
\begin{proof}
The derivative of the logistic sigmoid function $\sigma'(z) = (e^z + e^{-z} + 2)^{-1}$ achieves its global maximum at $z=0$ with $\sigma'(0) = 1/4$. Thus, the lower bound $\frac{1}{4} - \mathbb{E}[\sigma'(Z)] \ge 0$ holds trivially. To establish the upper bound, we observe that the function $g(z) = \frac{1}{4} - \sigma'(z)$ satisfies $g(0) = 0$, $g'(0) = 0$, and its second derivative is globally bounded above by $1/8$. Integrating the second derivative yields the global parabolic upper bound
\begin{equation}
\frac{1}{4} - \sigma'(z) \le \frac{z^2}{16}
\end{equation}
Taking the expectation over the Gaussian measure of $Z$ directly provides the upper bound $\frac{1}{16}\mathbb{E}[Z^2] = \frac{\sigma_h^2}{16}$.
\end{proof}

Let us now quantify the initial function space variance $\sigma_h^2$ under standard parameterizations. Consider a fully connected neural network architecture of depth $L$ with input dimension $d_0$ and hidden widths $d_l$ for $l \in \{1, \dots, L\}$. 

Let $x \in \mathbb{R}^{d_0}$ represent the input vector. The forward pass is defined recursively
\begin{equation}
z^{(l)} = W^{(l)} a^{(l-1)}
\end{equation}
\begin{equation}
a^{(l)} = \phi(z^{(l)})
\end{equation}
where $a^{(0)} = x$, $W^{(l)} \in \mathbb{R}^{d_l \times d_{l-1}}$ is the weight matrix at layer $l$, and $\phi(\cdot)$ is the element-wise activation function. The final scalar output is given by the linear transformation
\begin{equation}
h(x; \mathbf{W}_0) = W^{(L+1)} a^{(L)}
\end{equation}
with $W^{(L+1)} \in \mathbb{R}^{1 \times d_L}$.

\begin{proposition}
Assume the network weights are initialized according to the He initialization protocol for the ReLU activation function $\phi(z) = \max(0, z)$. Specifically $W_{ij}^{(l)} \sim \mathcal{N}\left(0, \frac{2}{d_{l-1}}\right)$ for $l \in \{1, \dots, L\}$ and the final linear layer is initialized as $W_{1j}^{(L+1)} \sim \mathcal{N}\left(0, \frac{1}{d_L}\right)$. The initial variance of the network output is exactly
\begin{equation}
\sigma_h^2 = \frac{\|x\|_2^2}{d_0}
\end{equation}
\end{proposition}

\begin{proof}
We proceed by induction on the second moment of the activations. For the first hidden layer the pre-activation components are $z_i^{(1)} = \sum_{j=1}^{d_0} W_{ij}^{(1)} x_j$. Since $W_{ij}^{(1)}$ are independent zero-mean Gaussian random variables we have
\begin{equation}
\mathbb{E}\left[\left(z_i^{(1)}\right)^2\right] = \sum_{j=1}^{d_0} \mathbb{E}\left[\left(W_{ij}^{(1)}\right)^2\right] x_j^2 = \frac{2}{d_0} \sum_{j=1}^{d_0} x_j^2 = \frac{2\|x\|_2^2}{d_0}
\end{equation}
Given that $z_i^{(1)}$ follows a symmetric zero-mean normal distribution the ReLU activation halves the expected squared magnitude
\begin{equation}
\mathbb{E}\left[\left(a_i^{(1)}\right)^2\right] = \frac{1}{2}\mathbb{E}\left[\left(z_i^{(1)}\right)^2\right] = \frac{\|x\|_2^2}{d_0}
\end{equation}
Assume the inductive hypothesis holds for layer $l-1$ such that $\mathbb{E}\left[\left(a_j^{(l-1)}\right)^2\right] = \frac{\|x\|_2^2}{d_0}$. For layer $l$ the weights and previous activations are independent
\begin{equation}
\mathbb{E}\left[\left(z_i^{(l)}\right)^2\right] = \sum_{j=1}^{d_{l-1}} \mathbb{E}\left[\left(W_{ij}^{(l)}\right)^2\right] \mathbb{E}\left[\left(a_j^{(l-1)}\right)^2\right] = d_{l-1} \left(\frac{2}{d_{l-1}}\right) \left(\frac{\|x\|_2^2}{d_0}\right) = \frac{2\|x\|_2^2}{d_0}
\end{equation}
Applying the ReLU properties again yields $\mathbb{E}\left[\left(a_i^{(l)}\right)^2\right] = \frac{\|x\|_2^2}{d_0}$. By induction this holds for the final hidden layer $L$. The scalar output evaluates the linear combination with the final weight vector
\begin{equation}
\sigma_h^2 = \mathbb{E}\left[\left(h(x; \mathbf{W}_0)\right)^2\right] = \sum_{j=1}^{d_L} \mathbb{E}\left[\left(W_{1j}^{(L+1)}\right)^2\right] \mathbb{E}\left[\left(a_j^{(L)}\right)^2\right]
\end{equation}
Substituting the variance of the final layer $W_{1j}^{(L+1)} \sim \mathcal{N}\left(0, \frac{1}{d_L}\right)$ directly yields
\begin{equation}
\sigma_h^2 = d_L \left(\frac{1}{d_L}\right) \left(\frac{\|x\|_2^2}{d_0}\right) = \frac{\|x\|_2^2}{d_0}
\end{equation}
\end{proof}

A similar result holds for the Xavier initialization under which $W_{ij}^{(l)} \sim \mathcal{N}\left(0, \frac{1}{d_{l-1}}\right)$ and $\phi(z) = z$ corresponding to the linear regime of symmetric activations near the origin. In this case, a simpler argument leveraging that the variance propagation trivially preserves the second moment at each layer yields $\sigma_h^2 = \frac{\|x\|_2^2}{d_0}$.

We can now substitute this exact variance derivation into the theoretical bound established previously to obtain
\begin{equation}
0 \le \frac{1}{4} - \mathbb{E}_{\mathbf{W}_0}[\sigma'(-y_i h(x_i; \mathbf{W}_0))] \le \frac{\|x_i\|_2^2}{16 d_0}.
\end{equation}

This bound thus demonstrates that the quality of the linearization approximation is  dictated by the input dimension. In standard computer vision datasets the input space dimension $d_0$ is adequately large (around 50K for the ones used in our paper). Assuming a data preprocessing that standardizes the input vectors, i.e. $\|x_i\|_2^2 = 1$, the absolute error scales  as $\mathcal{O}(d_0^{-1})$. The error term thus converges to zero as $d_0 \to \infty$ certifying the soundness of evaluating the expected acceleration dynamics using $\sigma'(0) = 1/4$. This calculation then justifies the approximation at $t=0$ we leveraged in our original theoretical result.

We proceed to address the degradation of this approximation as the optimization progresses for $t > 0$ by characterizing the Jensen's gap (which is an intuitive approach to compare $\mathbb{E}[\sigma'(Z_t)]$ vs. $\sigma'(\mathbb{E}[Z_t])$ when applying the Jensen's inequality). As the model fits the training data, the expected margin diverges  from zero. Let $\mu_t = \mathbb{E}[-y_i h(x_i; \mathbf{W}_t)]$. For samples successfully classified by the network, $\mu_t$ diverges towards $-\infty$.

\begin{proposition}
Assume the logit $Z_t = -y_i h(x_i; \mathbf{W}_t)$ maintains a Gaussian distribution $\mathcal{N}(\mu_t, \sigma_h^2)$ during gradient flow. As the expected margin grows such that $\mu_t \to -\infty$, the Jensen gap between the expected acceleration and the point estimate converges  to a constant multiplicative factor dictated by the variance
\begin{equation}
\lim_{\mu_t \to -\infty} \frac{\mathbb{E}[\sigma'(Z_t)]}{\sigma'(\mathbb{E}[Z_t])} = \exp \left( \frac{\sigma_h^2}{2} \right)
\end{equation}
\end{proposition}
\begin{proof}
The derivative of the logistic sigmoid function is $\sigma'(z) = \frac{e^z}{(1+e^z)^2}$. We express the random variable as $Z_t = \mu_t + \sigma_h X$ where $X \sim \mathcal{N}(0, 1)$. Then
\begin{equation}
\frac{\mathbb{E}[\sigma'(Z_t)]}{\sigma'(\mu_t)} = \frac{\mathbb{E}\left[ \frac{e^{\mu_t + \sigma_h X}}{(1 + e^{\mu_t + \sigma_h X})^2} \right]}{\frac{e^{\mu_t}}{(1+e^{\mu_t})^2}} = \mathbb{E}\left[ e^{\sigma_h X} \frac{(1+e^{\mu_t})^2}{(1 + e^{\mu_t + \sigma_h X})^2} \right]
\end{equation}
Let $g(\mu_t, X) = e^{\sigma_h X} \frac{(1+e^{\mu_t})^2}{(1 + e^{\mu_t + \sigma_h X})^2}$. For any fixed $X \in \mathbb{R}$, taking the limit yields
\begin{equation}
\lim_{\mu_t \to -\infty} g(\mu_t, X) = e^{\sigma_h X} \frac{(1+0)^2}{(1+0)^2} = e^{\sigma_h X}
\end{equation}
We need to  pass the limit inside the expectation. For all $\mu_t \le 0$ and $X \in \mathbb{R}$, we have $(1+e^{\mu_t})^2 \le 4$ and  $(1 + e^{\mu_t + \sigma_h X})^2 \ge 1$. Consequently,
\begin{equation}
|g(\mu_t, X)| \le 4 e^{\sigma_h X}
\end{equation}
The dominating function $4 e^{\sigma_h X}$ is integrable with respect to the standard Gaussian measure, as its expectation corresponds to a scaled log-normal moment
\begin{equation}
\mathbb{E}[4 e^{\sigma_h X}] = 4 \exp \left( \frac{\sigma_h^2}{2} \right) < \infty
\end{equation}
By Lebesgue's Dominated Convergence Theorem, the limit of the expectation is
\begin{equation}
\lim_{\mu_t \to -\infty} \mathbb{E}[g(\mu_t, X)] = \mathbb{E}\left[ \lim_{\mu_t \to -\infty} g(\mu_t, X) \right] = \mathbb{E}[e^{\sigma_h X}]
\end{equation}
Evaluating the moment generating function of the standard normal distribution at $\sigma_h$ completes the proof:
\begin{equation}
\mathbb{E}[e^{\sigma_h X}] = \exp \left( \frac{\sigma_h^2}{2} \right).
\end{equation}
\end{proof}

Given that we argued $\sigma_h^2 = \mathcal{O}(\frac{\|x\|_2^2}{d_0})$ previously, we expect the approximation does not degrade significantly as training progresses.

The requirements for the above proposition are further justified by the following result.

\begin{proposition}
The random variable $\sigma'(Z)$ exhibits sub-Gaussian concentration around its mean
\begin{equation}
\mathbb{P} \left( \big| \sigma'(Z) - \mathbb{E}[\sigma'(Z)] \big| \ge \delta \right) \le 2 \exp \left( - \frac{54 \delta^2}{\sigma_h^2} \right)
\end{equation}
\end{proposition}
\begin{proof}
The second derivative of the sigmoid function $\sigma''(z) = \sigma(z)(1-\sigma(z))(1-2\sigma(z))$ admits a global supremum norm $L = \sup_{z \in \mathbb{R}} |\sigma''(z)| = \frac{1}{6\sqrt{3}}$. Therefore, the function $\sigma': \mathbb{R} \to \mathbb{R}$ is uniformly $L$-Lipschitz continuous. Applying the Gaussian Lipschitz concentration inequality (see Chapter 2.3 and Theorem 2.26 in the book ``High-Dimensional Statistics A Non-Asymptotic Viewpoint'' by Martin J. Wainwright) for the function $\sigma'$ evaluated on the Gaussian random variable $Z$ yields the stated exponential tail bound with constant $(2L^2)^{-1} = 54$.
\end{proof}

\subsubsection{Verification of Simplicity Bias Through a Toy Dataset}

Our theoretical analysis on the simplicity bias condition relies on the simplicity bias condition $\beta^s(2\alpha-1) > \beta^c$ which conflates feature strength with data imbalance. To demonstrate how minority group learning dynamics change under varying $\beta^s/\beta^c$ and $\alpha$, we conduct experiments on a toy setup using a synthetic dataset that satisfies our theoretical setting and assumptions. We visualize the learning behavior of a linear model trained with gradient descent under the same setting used in our theoretical analysis.

We use a synthetic dataset $\mathcal{D}$ to visualize minority losses while varying $\beta_s/\beta_c$ and $\alpha$. We adopt the setting in our theoretical study: each sample has label $y\in\{-1,+1\}$ and spurious attribute $a\in\{-1,+1\}$ satisfying $\alpha=P(a=y)$. The input is
\[
x=
\begin{bmatrix}
\beta_c y\\
\beta_s a
\end{bmatrix},
\]
where $\beta_c$ is the core feature strength and $\beta_s$ is the spurious feature strength. This creates majority groups with $a=y$ and minority groups with $a\ne y$. 
We train a linear model $\hat{y}=x^\top w=w_c\beta_c y+w_s\beta_s a$ with full-batch gradient descent while varying $\beta_s/\beta_c$ and $\alpha$. After training, we measure the peak loss among all minority samples. Figure \ref{fig:toy_learning_setup} shows a consistent transition in learned behavior depending on whether $\beta_s(2\alpha-1)>\beta_c$ holds, and it also captures the boundary cases, verifying the simplicity bias condition.

\begin{figure}
    \centering
    \includegraphics[width=\linewidth]{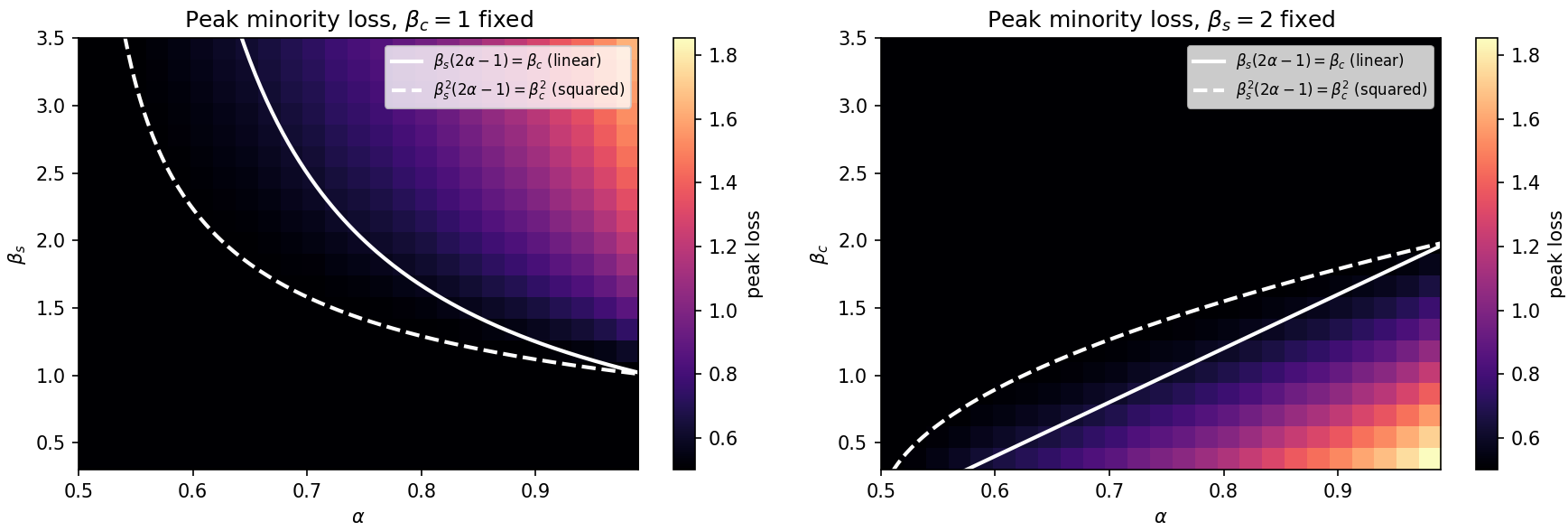}
    \caption{Learning behavior and simplicity bias analysis on a synthetic toy dataset.}
    \label{fig:toy_learning_setup}
\end{figure}

%%%% RELAXATION TO MULTIPLE SPURIOUS CORRELATIONS

\subsubsection{Relaxation to Multiple Spurious Correlations}

Here we extend the theoretical analysis of the homogeneous spiked model to accommodate a dataset where each sample is associated with multiple distinct spurious attributes. Let the training dataset $\mathcal{D}$ of size $n$ consist of samples $(x_i, y_i)$ where $y_i \in \{\pm 1\}$ represents the target label. Each sample $x_i$ contains a core feature vector $x_i^c$ and $M$ independent spurious feature vectors $x_i^{s,m}$ for $m \in \{1, \dots, M\}$. We denote the corresponding unobserved binary spurious attributes by the vector $\mathbf{a}_i \in \{\pm 1\}^M$ where the $m$-th element is $a_{i,m}$. We quantify the consistency of each spurious attribute within the dataset by defining the scalar $\alpha_m \in (0.5, 1]$ as the fraction of samples satisfying $a_{i,m} = y_i$.

\begin{assumption}
The neural network architecture permits an additive decomposition of the output logit into a core subnetwork and $M$ distinct spurious subnetworks
\begin{equation}
h(x; \mathbf{W}) = h_c(x^c; \mathbf{W}^c) + \sum_{m=1}^M h_{s,m}(x^{s,m}; \mathbf{W}^{s,m})
\end{equation}
where the parameter sets $\mathbf{W}^c$ and $\mathbf{W}^{s,m}$ for $m \in \{1, \dots, M\}$ are mutually disjoint.
\end{assumption}

By the linearity of the gradient with respect to disjoint parameter subsets, the NTK decomposes additively into $M+1$ distinct components evaluated on their respective feature subspaces
\begin{equation}
K(x, x') = K_c(x^c, x'^c) + \sum_{m=1}^M K_{s,m}(x^{s,m}, x'^{s,m})
\end{equation}
where $K_c$ and $K_{s,m}$ represent the tangent kernels of the core and spurious subnetworks.

\begin{assumption}
The kernel components exhibit a uniform rank-1 spiked structure corresponding to the latent data generation process
\begin{equation}
K_c(x_j^c, x_i^c) = \beta^c y_j y_i
\end{equation}
and
\begin{equation}
K_{s,m}(x_j^{s,m}, x_i^{s,m}) = \beta^{s,m} a_{j,m} a_{i,m}
\end{equation}
where the scalars $\beta^c > 0$ and $\beta^{s,m} > 0$ capture the inherent feature strengths of the core and spurious components respectively.
\end{assumption}

We assume symmetric initialization such that the expected initial output of the network is zero. Under gradient flow with the logistic loss function, the exact initial dynamics of the expected logits are governed by the dataset statistics and the isolated kernel strengths.

\begin{theorem}
Let $\bar{h}_t(x_j) = \mathbb{E}_{\mathbf{W}_0}[h(x_j; \mathbf{W}_t)]$ denote the expected logit for sample $j$ at time $t$. The initial expected velocities of the subnetworks are given by
\begin{equation}
\left. \frac{\partial \bar{h}_t^c(x_j)}{\partial t} \right|_{t=0} = \frac{\beta^c}{2} y_j
\end{equation}
and
\begin{equation}
\left. \frac{\partial \bar{h}_t^{s,m}(x_j)}{\partial t} \right|_{t=0} = \frac{\beta^{s,m} (2\alpha_m - 1)}{2} a_{j,m}
\end{equation}
for all $m \in \{1, \dots, M\}$.
\end{theorem}

\begin{proof}
The time derivative of the expected logit evaluated at $t=0$ under the specified flow ordinary differential equation is $\dot{\bar{\mathbf{h}}}_0 = \frac{1}{2n} K \mathbf{y}$. We project the vector $\mathbf{y} \in \mathbb{R}^n$ onto the respective decomposed kernel matrices. For the core subnetwork we have
\begin{equation}
\left. \frac{\partial \bar{\mathbf{h}}_t^c}{\partial t} \right|_{t=0} = \frac{1}{2n} (\beta^c \mathbf{y} \mathbf{y}^\top) \mathbf{y} = \frac{\beta^c}{2n} \mathbf{y} (\mathbf{y}^\top \mathbf{y}) = \frac{\beta^c}{2} \mathbf{y}
\end{equation}
since $\mathbf{y}^\top \mathbf{y} = \sum_{i=1}^n y_i^2 = n$. For the $m$-th spurious subnetwork we project $\mathbf{y}$ onto the rank-1 spurious matrix parameterized by the vector $\mathbf{a}_m \in \mathbb{R}^n$
\begin{equation}
\left. \frac{\partial \bar{\mathbf{h}}_t^{s,m}}{\partial t} \right|_{t=0} = \frac{1}{2n} (\beta^{s,m} \mathbf{a}_m \mathbf{a}_m^\top) \mathbf{y} = \frac{\beta^{s,m}}{2n} \mathbf{a}_m (\mathbf{a}_m^\top \mathbf{y})
\end{equation}
The inner product evaluates to $\mathbf{a}_m^\top \mathbf{y} = \sum_{i=1}^n a_{i,m} y_i$. By the definition of $\alpha_m$, the number of samples where $a_{i,m} = y_i$ is $\alpha_m n$ and the number where $a_{i,m} = -y_i$ is $(1-\alpha_m)n$. Hence the summation equals $n(2\alpha_m - 1)$. Substitution yields the exact velocity element-wise.
\end{proof}

We formalize the generalized simplicity bias phenomenon by evaluating the initial derivative of the true expected margin $\bar{m}_t(x_j) = \mathbb{E}_{\mathbf{W}_0}[y_j h(x_j; \mathbf{W}_t)]$. The evolution of the margin demonstrates how intersecting group alignments either accelerate or hinder the correct classification of any sample $j$.

\begin{theorem}
Define the disjoint index sets $\mathcal{M}_{align}(j) = \{m \mid y_j a_{j,m} = 1\}$ and $\mathcal{M}_{conflict}(j) = \{m \mid y_j a_{j,m} = -1\}$. The expected margin for sample $j$ initially decreases and the corresponding cross-entropy loss  increases if and only if
\begin{equation}
\sum_{m \in \mathcal{M}_{conflict}(j)} \beta^{s,m}(2\alpha_m - 1) > \beta^c + \sum_{m \in \mathcal{M}_{align}(j)} \beta^{s,m}(2\alpha_m - 1)
\end{equation}
\end{theorem}

\begin{proof}
We compute the initial time derivative of the expected margin by multiplying the total initial logit velocity by $y_j$
\begin{equation}
\left. \frac{\partial \bar{m}_t(x_j)}{\partial t} \right|_{t=0} = y_j \left( \frac{\beta^c}{2} y_j + \sum_{m=1}^M \frac{\beta^{s,m}(2\alpha_m - 1)}{2} a_{j,m} \right)
\end{equation}
Applying $y_j^2 = 1$ and distributing $y_j$ into the summation yields
\begin{equation}
\left. \frac{\partial \bar{m}_t(x_j)}{\partial t} \right|_{t=0} = \frac{\beta^c}{2} + \sum_{m=1}^M \frac{\beta^{s,m}(2\alpha_m - 1)}{2} y_j a_{j,m}
\end{equation}
We partition the summation over the sets $\mathcal{M}_{align}(j)$ and $\mathcal{M}_{conflict}(j)$ where $y_j a_{j,m}$ takes values of $1$ and $-1$ respectively
\begin{equation}
\left. \frac{\partial \bar{m}_t(x_j)}{\partial t} \right|_{t=0} = \frac{\beta^c}{2} + \sum_{m \in \mathcal{M}_{align}(j)} \frac{\beta^{s,m}(2\alpha_m - 1)}{2} - \sum_{m \in \mathcal{M}_{conflict}(j)} \frac{\beta^{s,m}(2\alpha_m - 1)}{2}
\end{equation}
The loss function  increases at initialization if and only if the margin derivative is  negative. Setting the right-hand side to be less than zero directly yields the stated condition.
\end{proof}

%%%%% BACKGROUND INFORMATION

\subsection{Background on the NTK and Spiked Covariance Model}

Below, we provide a brief introduction to NTK and spiked covariance models to make our theoretical analysis more accessible.

\subsubsection{Background on the NTK}
\begin{definition}
Let $h(x; \mathbf{W})$ denote the scalar output of a neural network parameterized by the weight vector $\mathbf{W} \in \mathbb{R}^P$. The NTK evaluates the inner product of the gradients of the network output with respect to its parameters evaluated at two inputs $x$ and $x'$
\begin{equation}
K(x, x'; \mathbf{W}) = \nabla_{\mathbf{W}} h(x; \mathbf{W})^\top \nabla_{\mathbf{W}} h(x'; \mathbf{W})
\end{equation}
\end{definition}

The key idea of the NTK framework is that as the width of the hidden layers approaches infinity, this empirical kernel converges to a fixed limit.

\begin{theorem}
In the infinite-width limit, under symmetric initialization of the parameters $\mathbf{W}_0$, the empirical kernel $K(x, x'; \mathbf{W}_t)$ converges in probability to a deterministic,  positive semi-definite kernel $K(x, x')$. Furthermore, this kernel remains time-invariant throughout the optimization process.
\end{theorem}

This time-invariance property implies that the highly non-linear parameter space optimization maps to a linear functional space optimization. To formalize the learning dynamics, one typically approximates discrete stochastic gradient descent via continuous-time gradient flow, as done in our paper as well. 

\begin{definition}
Let $\mathcal{L}(\mathbf{W}_t) = \frac{1}{n} \sum_{i=1}^n \ell(y_i, h(x_i; \mathbf{W}_t))$ define the empirical risk over a dataset of size $n$, where $\ell(\cdot, \cdot)$ is a differentiable loss function. Continuous-time gradient flow dictates the parameter evolution according to the ordinary differential equation
\begin{equation}
\frac{d \mathbf{W}_t}{d t} = -\nabla_{\mathbf{W}} \mathcal{L}(\mathbf{W}_t)
\end{equation}
\end{definition}

By applying the chain rule, we map the evolution of the network parameters to the evolution of the network outputs (logits) for any given input $x$.

\begin{theorem}
Under gradient flow in the NTK regime, the exact evolution of the expected network output $\bar{h}_t(x) = \mathbb{E}_{\mathbf{W}_0}[h(x; \mathbf{W}_t)]$ is governed by the deterministic equation
\begin{equation}
\frac{\partial \bar{h}_t(x)}{\partial t} = -\frac{1}{n} \sum_{i=1}^n K(x, x_i) \mathbb{E}_{\mathbf{W}_0} \left[ \frac{\partial \ell(y_i, h(x_i; \mathbf{W}_t))}{\partial h(x_i; \mathbf{W}_t)} \right]
\end{equation}
\end{theorem}

This differential equation constitutes the core machinery of our theoretical analysis. It demonstrates that the instantaneous change in the model's prediction for a sample $x$ is a linear combination of the gradients of the loss evaluated on all training samples $x_i$, weighted exactly by the similarity measure defined by the kernel $K(x, x_i)$.

To intuitively understand why this framework is necessary for our paper, consider the separation of features. If an architecture processes a core feature $x^c$ and a spurious feature $x^s$ through disjoint parameter subsets, the definition of the NTK ensures that the global kernel additively decomposes into a core kernel and a spurious kernel. 

Consequently, the differential equation governing the learning dynamics linearly separates into independent velocity components driven by these respective kernels. By substituting the spiked covariance model (discussed in the next section) into this ODE, we can extract the learning velocities of different features. Features that frequently co-occur with the target label yield large cumulative sums in the differential equation, forcing the network to minimize the loss along those feature dimensions at a faster rate. This formulation thus helps us avoid heuristic explanations of simplicity bias and instead, to our knowledge for the first time, concretely quantify the conditions under which a neural network prioritizes spurious correlations.

\subsubsection{Background on the Spiked Covariance Model}
Characterizing the exact optimization dynamics of overparameterized neural networks requires analyzing the spectrum of the data covariance or the induced Gram matrix. When input data lacks specific latent structures, the sample covariance spectrum is bounded and continuously distributed, rendering the isolation of individual feature learning velocities impossible. Accordingly, we propose relying on the widely adopted notion of spiked models.

\begin{definition}
The spiked covariance model posits that the population covariance matrix $\Sigma \in \mathbb{R}^{d \times d}$ decomposes into an isotropic background noise component perturbed by a low-rank structural matrix of rank $K \ll d$
\begin{equation}
\Sigma = \sigma^2 I_d + \sum_{k=1}^K \lambda_k v_k v_k^\top
\end{equation}
where $\sigma^2 > 0$ denotes the uniform noise variance, $\lambda_k > 0$ are the discrete spike eigenvalues representing signal strengths, and $v_k \in \mathbb{R}^d$ are orthonormal vectors defining the principal latent directions of the data distribution.
\end{definition}

In the context of machine learning, these principal directions encode the dominant predictive features embedded within the input space. When evaluating the NTK across the training dataset, the inner products of the network gradients are overwhelmingly governed by these underlying latent factors.

In our paper, we map this classical statistical model directly to the kernel matrix induced by the NTK. We assume the dataset generation is  governed by dominant latent variables corresponding to the core and spurious features. We study the idealized low-rank structure of the kernel matrix $K \in \mathbb{R}^{n \times n}$ by isolating the spectral spikes
\begin{equation}
K = \beta^c \mathbf{y} \mathbf{y}^\top + \beta^s \mathbf{a} \mathbf{a}^\top
\end{equation}
The orthonormal vectors $v_k$ from the classical statistical formulation are replaced by the dataset-level structural vectors $\mathbf{y}$ and $\mathbf{a}$, while the spike eigenvalues $\lambda_k$ map to the isolated feature strengths $\beta^c$ and $\beta^s$.

This framework thus helps bridge random matrix theory and deep learning optimization, which we found essential for our ensuing theoretical analysis: It guarantees that the gradient flow dynamics are constrained to a low-dimensional subspace spanned precisely by the core and spurious features. Without adopting the spiked covariance formulation, the continuous-time ordinary differential equations governing the expected logits would entangle across all $n$ data dimensions. By restricting the kernel to have discrete structural spikes, we decouple the dataset complexity and derive exact, closed-form velocities for the competing feature components.

%%%%%

%% file: appendix/sec_b.tex
\section{Memorization and CSL Definitions}\label{sec:app_memo_csl}

\subsection{Memorization Score}
The memorization score introduced by \cite{Memorization} is quantitatively defined and calculated for each training sample as follows:
\begin{equation}
\begin{split}
    \text{mem}(f, \mathcal{D}, i) &= \mathbb{P}_{\hat{f} = f(\mathcal{D})} \big[ \argmax(\hat{f}(x_i)) = y_i\big]
     - \mathbb{P}_{\hat{f} = f(\mathcal{D}^{\setminus i})} \big[\argmax(\hat{f}(x_i)) = y_i\big],
\end{split}
\end{equation}
Here, $f$ denotes the learning algorithm trained on the dataset, and $\hat{f}$ represents the trained model. $\mathcal{D}^{\backslash i}$ denotes the subportion of the dataset $\mathcal{D}$ with the sample $i$ removed. While memorization score has proven effective for various tasks including coreset selection~\citep{SelfSup}, its computation is prohibitively expensive, requiring the model to be retrained from scratch after removing each individual sample. Consequently, researchers have proposed several computationally efficient proxies.

\subsection{Cumulative Sample Loss (CSL)}
The Cumulative Sample Loss (CSL) was recently proposed by \cite{CSL} as an efficient proxy for the memorization score and is computed as
\begin{equation}
\text{CSL}(x_i) = \frac{1}{T}\sum_{t=1}^{T}\ell\big(y_i,f(x_i,\mathbf{W}_t)\big),
\end{equation}
where the cross-entropy loss $\ell(\cdot)$ is calculated for each training sample $x_i$. Sample losses are recorded at the end of every training epoch and the final CSL is obtained by averaging the losses across all training epochs.

%% file: appendix/sec_c.tex
\section{Formal Definitions of Helper Algorithms}\label{sec:app_helper_algs}

Helper algorithms used in the main paper are given in \textbf{Algorithms \ref{alg:wkmeans}, \ref{alg:select-bottom} and \ref{alg:select-hist}}.

\begin{algorithm}[t]
\caption{$\textproc{wKMEANS}$}
\label{alg:wkmeans}
\textbf{Input:}  weights $\mathbf{w}=\{\mathrm{w}_i\}_{i=1}^n$, sample representations $\mathbf{L}=\{\mathbf{L}_i\}_{i=1}^n$, number of clusters $K\!=\!2$. \\
\textbf{Output:} Cluster assignment sets $\{G_k\}_{k=1}^K$.
\begin{algorithmic}[1]
\State \textbf{Initialization:} Initialize centers $\{\mu_k\}_{k=1}^K$ randomly from $\mathbf{L}$. Set maximum iterations to $200$.
\While{cluster assignments $\{c_i\}$ change and the maximum iterations has not reached}
    \State \textbf{Assignment step:}
    \For{$i = 1$ \textbf{to} $n$}
        \State Assign $x_i$ to the closest center
        \[
            c_i \leftarrow \arg\min_{k \in \{1,\dots,K\}} \|\mathbf{L}_i - \mu_k\|_2^2
        \]
    \EndFor
    \State Define clusters $G_k \leftarrow \{i\!:\!c_i\!=\!k\!\}$ for $k=1,\dots,K$.
    \State \textbf{Update step:}
    \For{$k = 1$ \textbf{to} $K$}
        \State Update center using weighted mean:
        \[
            \mu_k \leftarrow 
            \frac{\sum_{i \in G_k} \mathrm{w}_i \mathbf{L}_i}{\sum_{i \in G_k} \mathrm{w}_i}.
        \]
    \EndFor
\EndWhile
\State \textbf{Return} clusters $\{G_k\}_{k=1}^K$.
\end{algorithmic}
\end{algorithm}

\begin{algorithm}[t]
\caption{$\textproc{SelectBot}$}
\label{alg:select-bottom}
\textbf{Input:} Sorted group $\tilde{G}$ by TCSL$_c$ (ascending), number of samples to select $\tilde{n}$.\\
\textbf{Output:} Selected set $\tilde{\mathcal{D}}$.
\begin{algorithmic}[1]
\State \textbf{Return} $\tilde{\mathcal{D}} \gets \tilde{G}[1\!:\!\tilde{n}]$
\end{algorithmic}
\end{algorithm}

\begin{algorithm}[t]
\caption{$\textproc{SelectHist}$}
\label{alg:select-hist}
\textbf{Input:} Sorted group $\tilde{G}$ by TCSL$_c$, number of samples to select $\tilde{n}$, number of bins $B$.\\
\textbf{Output:} Selected set $\tilde{\mathcal{D}}$.
\begin{algorithmic}[1]
\State Partition $\tilde{G}$ into $B$ equal-size bins $\{\tilde{G}_1,\ldots,\tilde{G}_B\}$
\State $\tilde{\mathcal{D}} \gets \emptyset$
\While{$|\tilde{\mathcal{D}}| < \tilde{n}$}
    \For{each bin $\tilde{G}_b$}
        \If{$|\tilde{\mathcal{D}}| = \tilde{n}$}      \State \textbf{break} 
        \EndIf
        \If{$\tilde{G}_b \neq \emptyset$}
            \State pick a random $\tilde{x} \in \tilde{G}_b$
            \State $\tilde{\mathcal{D}} \gets \tilde{\mathcal{D}} \cup \{\tilde{x}\}$
            \State $\tilde{G}_b \gets \tilde{G}_b \setminus \{\tilde{x}\}$
        \EndIf
    \EndFor
\EndWhile
\State \textbf{Return} $\tilde{\mathcal{D}}$
\end{algorithmic}
\end{algorithm}

%% file: appendix/sec_d.tex
\section{Dataset Details}\label{sec:app_dataset_details}

We present detailed explanations of the datasets used in our experiments. Further information regarding the groups is illustrated in Table \ref{tab:dataset_details}. \\

\noindent \textbf{Waterbirds}~\citep{Waterbirds} The labels are landbird and waterbird, where the spurious attribute corresponds to the background type: land or water. The dataset is synthetically constructed by placing bird images from the Caltech-UCSD Birds-200-2011 dataset~\citep{CaltechBirds} onto background images from the Places dataset~\citep{Places}. \\

\noindent \textbf{cMNIST}~\citep{Cmnist} A synthetic variant of the MNIST dataset consisting of $10$ digit classes, where each digit is assigned a distinct color that serves as the spurious attribute. \\

\noindent \textbf{MetaShift}~\citep{Metashift} This dataset contains cat and dog images, where the spurious attribute corresponds to the environment of the animal, given as indoor or outdoor. \\

\noindent \textbf{UrbanCars-B}~\citep{Urbancars} The goal is to classify samples as urban or country cars, where the background serves as the spurious attribute. The dataset is synthetically constructed by placing car images onto background images from the Places dataset~\citep{Places}.

\begin{table}[t]
\centering
\begin{tabular}{lcccc}
\toprule
Dataset & $y$ & $a=y$ & $a=-y$ & $\alpha$ \\
\midrule
Waterbirds~\citep{Waterbirds} & 2 & 4555 & 240 & 0.950 \\
cMNIST~\citep{Cmnist} & 10 & 52551 & 257 & 0.995 \\
MetaShift~\citep{Metashift} & 2 & 1500 & 300 & 0.882 \\
UrbanCars-B~\citep{Urbancars} & 2 & 7600 & 400 & 0.950 \\
\bottomrule
\end{tabular}
\caption{Dataset details including the number of classes, number of samples where the spurious attribute agrees with the label, number of samples where the spurious attribute does not agree with the label and the $\alpha$ parameter.}
\label{tab:dataset_details}
\end{table}

%% file: appendix/sec_e.tex
\section{Baseline Methods and Sample Scores}\label{sec:app_info_baseline}
In this section, we describe the baseline debiasing methods used in our experiments, as well as the sample scoring functions employed in the coreset selection experiments.

CB ERM and GB ERM denote the standard ERM model trained with samples reweighted by the inverse of their class sizes and group sizes, respectively. GroupDRO \citep{GroupDRO} uses group labels during training to upweight samples from the worst performing group and directly minimizes the worst group loss. LC \citep{LC} is a two-stage training algorithm that uses information from a first model as prior knowledge about the groups to shift the second model towards a more group-balanced solution. DFR \citep{DFR} retrains the last layer of a standard ERM model, initially trained on the entire dataset, using a group-balanced validation set. CNC \citep{CNC} is a two-stage training algorithm in which the second model is trained with a contrastive loss that aligns representations of samples within a class while mitigating spurious correlations. LfF \citep{LFF} is a two-stage training algorithm where the debiased model is trained to upweight samples on which the biased model fails to predict accurately. JTT \citep{JTT} first trains an ERM model and marks the samples it misclassifies, then retrains another model from scratch while upweighting these misclassified samples. ULA \citep{uLA} is a two-stage training algorithm where the biased model is a self-supervised pretrained network and the debiased model is trained with a logit adjustment similar to LC. EIIL \citep{EIIL} learns an invariant model based on groups (environments) identified by a reference model. GEORGE \citep{GEORGE} clusters the feature space of a standard ERM model within each class to identify groups and then uses the inferred groups to train a new model with an objective similar to GroupDRO.

For our coreset selection experiments, following the setting in \citep{CS_LR}, we employ four baseline scoring functions to compute sample scores. EL2N \citep{EL2N} assigns a difficulty score to each sample based on the norm of the difference between the model’s predicted probability vector and the one-hot label vector. SelfSup \citep{SelfSup} uses an embedding-based scoring method, defined as the norm of the difference between the feature vector of a sample and the mean feature vector of its assigned cluster. Random \citep{CS_LR} randomly selects samples according to the coreset selection ratio $r$. RGbal \citep{CS_LR} uses group labels to always select samples from the minority group and randomly selects from the majority group to satisfy the remaining quota. We additionally include the state-of-the-art coreset selection algorithm D2~\cite{D2} in our experiments. D2 combines a chosen sample scoring function with feature embeddings to capture feature similarity and employs a message-passing scheme to select coresets with high distributional coverage. As the sample scoring function of D2, we use the EL2N and SelfSup scores. For feature embeddings, we consider (i) representations extracted from a ResNet model trained with ERM on the corresponding training dataset and (ii) representations from the pretrained CLIP model~\cite{CLIP}. This results in four variants of the D2 algorithm.

%% file: appendix/sec_f.tex
\section{Hyperparameters for Model Training and Coreset Selection}\label{sec:app_hyperparams}

\subsection{Model Training}
To maintain consistency with prior work, we use SGD as our optimization algorithm. Following the literature, we adopt the ResNet50 architecture for Waterbirds, MetaShift and UrbanCars-B, and ResNet18 for cMNIST. All architectures are initialized with ImageNet-1K \citep{Imagenet} pretrained weights. The hyperparameters used for the ERM model in our experiments are listed in Table~\ref{table:hyperparams}. Both the model architectures and hyperparameters are chosen to match the most commonly employed configurations in baseline studies; therefore, we do not perform hyperparameter tuning or early stopping. For both $f_s$ and $f_c$, we use the same hyperparameters as the ERM model, changing only the total training epochs to $T_s\!=\!T/10$ for the spurious network. Our goal is to demonstrate that the performance gains achieved by retraining the ERM model on the coreset selected by TCSL-CS arise solely from the effectiveness of our proposed coreset selection method, rather than from any hyperparameter adjustments.

\begin{table}[t]
\centering
\small
\setlength{\tabcolsep}{4pt}
\begin{tabular}{lcccc}
\toprule
\textbf{Parameter} & \textbf{Waterbirds} & \textbf{cMNIST} & \textbf{MetaShift} & \textbf{UrbanCars-B} \\
\midrule
Learning rate    & 1e-4 & 1e-3 & 1e-3 & 1e-4 \\
Weight decay     & 1e-1 & 1e-3 & 1e-3 & 1e-1 \\
Momentum         & 0.9  & 0.9  & 0.9  & 0.9  \\
Batch size       & 128  & 32   & 32   & 128  \\
Training epochs  & 300  & 50   & 200  & 300  \\
\bottomrule
\end{tabular}
\caption{Hyperparameters for the datasets used in our experiments.}
\label{table:hyperparams}
\end{table}

\subsection{Coreset Selection}
Our coreset selection algorithm TCSL-CS has two hyperparameters: the number of bins $B$ used for histogram-based selection and the threshold $\tau$, which determines when to switch from bottom-based selection ($\textproc{SelectBot}$) to histogram-based selection ($\textproc{SelectHist}$). Staying consistent with the coreset selection literature, we treat the coreset selection ratio $r$ as a user-specified input to the algorithm. Following prior histogram-based coreset selection methods~\citep{HistCS}, we set $B\!=\!50$ by default and do not tune it. For all datasets except cMNIST, we set $\tau\!=\!0.4$. For cMNIST, due to the lower complexity of digit images, the $\text{TCSL}_c$ scores do not form a sufficiently diverse distribution. Therefore, we disable $\textproc{SelectHist}$ by setting $\tau\!=\!0$.

For retraining the CB ERM model on the selected coresets, we fix the total number of training iterations by setting the number of training epochs to $T/r$, where $r$ is the coreset selection ratio and $T$ denotes the total training epochs of the CB ERM model trained on the full dataset. We compute the EL2N \citep{EL2N} scores using the CB ERM model after $T/10$ epochs of training on the full dataset.

%% file: appendix/sec_g.tex
\section{Additional Experiments}\label{sec:app_add_exp}
We provide additional coreset selection results for different selection ratios. We demonstrate results for ratios $0.2$, $0.4$, $0.6$ and $0.8$ in Tables \ref{tab:tab_0.2}, \ref{tab:tab_0.4}, \ref{tab:tab_0.6} and \ref{tab:tab_0.8}, respectively.

\begin{table*}[!hbtp]
\centering
\resizebox{\textwidth}{!}{
\begin{tabular}{lccccc}
\toprule
\textbf{Method} & {\textbf{Group Info}} & \textbf{Waterbirds} & \textbf{cMNIST} & \textbf{MetaShift} & \textbf{UrbanCars-B} \\
& Train & WGA (AVG) & WGA (AVG) & WGA (AVG) & WGA (AVG) \\
\midrule
EL2N (Bot)     & x & $44.89${\scriptsize$\pm2.14$} ($93.98${\scriptsize$\pm0.42$}) & $0.00${\scriptsize$\pm0.00$} ($13.92${\scriptsize$\pm0.67$}) & $52.78${\scriptsize$\pm5.93$} ($77.95${\scriptsize$\pm1.71$}) & $34.00${\scriptsize$\pm5.23$} ($66.73${\scriptsize$\pm4.40$}) \\
EL2N (Top)     & x & $\underline{89.56}${\scriptsize$\pm0.93$} ($98.05${\scriptsize$\pm0.16$}) & $49.66${\scriptsize$\pm7.07$} ($75.06${\scriptsize$\pm2.09$}) & $77.55${\scriptsize$\pm2.23$} ($87.50${\scriptsize$\pm0.60$}) & $62.53${\scriptsize$\pm0.46$} ($74.87${\scriptsize$\pm1.34$}) \\
EL2N (Hist)    & x & $74.28${\scriptsize$\pm2.13$} ($96.57${\scriptsize$\pm0.07$}) & $10.40${\scriptsize$\pm5.71$} ($35.18${\scriptsize$\pm1.39$}) & $68.98${\scriptsize$\pm2.63$} ($84.14${\scriptsize$\pm0.86$}) & $57.07${\scriptsize$\pm3.35$} ($78.30${\scriptsize$\pm0.95$}) \\
SelfSup (Bot)       & x & $61.21${\scriptsize$\pm0.26$} ($90.98${\scriptsize$\pm0.15$}) & $0.00${\scriptsize$\pm0.00$} ($19.16${\scriptsize$\pm2.00$}) & $47.92${\scriptsize$\pm2.08$} ($81.02${\scriptsize$\pm0.70$}) & $43.47${\scriptsize$\pm2.95$} ($77.50${\scriptsize$\pm0.80$}) \\
SelfSup (Top)       & x & $83.85${\scriptsize$\pm1.80$} ($97.80${\scriptsize$\pm0.29$}) & $52.70${\scriptsize$\pm8.00$} ($76.05${\scriptsize$\pm1.28$}) & $48.38${\scriptsize$\pm6.42$} ($78.88${\scriptsize$\pm1.71$}) & $28.13${\scriptsize$\pm3.63$} ($68.63${\scriptsize$\pm1.10$}) \\
SelfSup (Hist)      & x & $76.88${\scriptsize$\pm0.34$} ($96.84${\scriptsize$\pm0.39$}) & $5.40${\scriptsize$\pm2.81$} ($34.16${\scriptsize$\pm0.91$}) & $67.59${\scriptsize$\pm2.44$} ($83.97${\scriptsize$\pm1.16$}) & $60.00${\scriptsize$\pm5.64$} ($80.17${\scriptsize$\pm1.69$}) \\
D2 (EL2N+ResNet) & x &
$\underline{89.56}${\scriptsize$\pm0.41$} ($98.26${\scriptsize$\pm0.01$}) &
$20.79${\scriptsize$\pm11.92$} ($42.64${\scriptsize$\pm8.86$}) &
$\underline{79.17}${\scriptsize$\pm0.69$} ($87.91${\scriptsize$\pm0.70$}) &
$70.00${\scriptsize$\pm2.23$} ($84.37${\scriptsize$\pm1.03$}) \\
D2 (SelfSup+ResNet) & x &
$83.84${\scriptsize$\pm0.40$} ($96.97${\scriptsize$\pm0.13$}) &
$25.57${\scriptsize$\pm7.09$} ($58.65${\scriptsize$\pm5.23$}) &
$48.84${\scriptsize$\pm6.85$} ($78.99${\scriptsize$\pm1.84$}) &
$41.33${\scriptsize$\pm1.97$} ($75.10${\scriptsize$\pm0.66$}) \\
D2 (EL2N+CLIP) & x &
$86.09${\scriptsize$\pm1.00$} ($98.22${\scriptsize$\pm0.03$}) &
$27.87${\scriptsize$\pm11.64$} ($56.75${\scriptsize$\pm4.07$}) &
$73.84${\scriptsize$\pm1.06$} ($86.86${\scriptsize$\pm0.61$}) &
$62.80${\scriptsize$\pm3.49$} ($80.20${\scriptsize$\pm1.35$}) \\
D2 (SelfSup+CLIP) & x &
$80.32${\scriptsize$\pm0.74$} ($97.72${\scriptsize$\pm0.03$}) &
$22.05${\scriptsize$\pm9.23$} ($50.58${\scriptsize$\pm11.16$}) &
$68.29${\scriptsize$\pm2.81$} ($83.97${\scriptsize$\pm0.70$}) &
$37.47${\scriptsize$\pm4.28$} ($72.07${\scriptsize$\pm0.23$}) \\
Random             & x & $73.44${\scriptsize$\pm2.31$} ($97.15${\scriptsize$\pm0.08$}) & $5.40${\scriptsize$\pm0.72$} ($34.60${\scriptsize$\pm1.39$}) & $67.36${\scriptsize$\pm1.39$} ($83.97${\scriptsize$\pm2.24$}) & $56.00${\scriptsize$\pm0.69$} ($78.97${\scriptsize$\pm0.86$})\\
RGbal              & \checkmark & $87.49${\scriptsize$\pm1.82$} ($96.98${\scriptsize$\pm0.15$}) & $\underline{73.83}${\scriptsize$\pm3.65$} ($85.58${\scriptsize$\pm1.10$}) & $\boldsymbol{80.79}${\scriptsize$\pm1.06$} ($87.85${\scriptsize$\pm1.22$}) & $\underline{81.73}${\scriptsize$\pm1.10$} ($87.60${\scriptsize$\pm0.10$}) \\
TCSL-CS           & x & $\boldsymbol{90.60}${\scriptsize$\pm0.48$} ($96.68${\scriptsize$\pm0.08$}) & $\boldsymbol{85.26}${\scriptsize$\pm1.72$} ($92.22${\scriptsize$\pm0.59$}) & $76.85${\scriptsize$\pm4.24$} ($88.60${\scriptsize$\pm1.12$}) & $\boldsymbol{83.60}${\scriptsize$\pm0.40$} ($88.57${\scriptsize$\pm0.38$}) \\
\bottomrule
\end{tabular}
}
\caption{We compare the WGA and AVG of different coreset selection methods at $r\!=\!0.2$ across datasets. Results are averaged over $3$ seeds. We use CB ERM for retraining on all of the identified coresets. The best WGA for each dataset is shown in \textbf{bold} and the second best value is \underline{underlined}.
}
\label{tab:tab_0.2}
\end{table*}

\begin{table*}[!hbtp]
\centering
\resizebox{\textwidth}{!}{
\begin{tabular}{lccccc}
\toprule
\textbf{Method} & {\textbf{Group Info}} & \textbf{Waterbirds} & \textbf{cMNIST} & \textbf{MetaShift} & \textbf{UrbanCars-B} \\
& Train & WGA (AVG) & WGA (AVG) & WGA (AVG) & WGA (AVG) \\
\midrule
EL2N (Bot)     & x & $35.57${\scriptsize$\pm4.43$} ($92.71${\scriptsize$\pm0.45$}) & $0.00${\scriptsize$\pm0.00$} ($14.21${\scriptsize$\pm0.04$}) & $60.42${\scriptsize$\pm7.05$} ($80.21${\scriptsize$\pm2.13$}) & $29.87${\scriptsize$\pm6.02$} ($65.63${\scriptsize$\pm4.49$}) \\
EL2N (Top)     & x & $85.25${\scriptsize$\pm0.39$} ($98.39${\scriptsize$\pm0.04$}) & $52.18${\scriptsize$\pm14.40$} ($75.43${\scriptsize$\pm4.29$}) & $75.46${\scriptsize$\pm2.12$} ($88.54${\scriptsize$\pm0.52$}) & $73.33${\scriptsize$\pm3.21$} ($85.67${\scriptsize$\pm1.10$}) \\
EL2N (Hist)    & x & $78.09${\scriptsize$\pm0.01$} ($97.46${\scriptsize$\pm0.04$}) & $25.76${\scriptsize$\pm8.73$} ($51.65${\scriptsize$\pm3.97$}) & $71.76${\scriptsize$\pm1.75$} ($85.19${\scriptsize$\pm0.40$}) & $62.40${\scriptsize$\pm2.80$} ($80.87${\scriptsize$\pm0.21$}) \\
SelfSup (Bot)       & x & $50.95${\scriptsize$\pm1.32$} ($89.60${\scriptsize$\pm0.03$}) & $0.00${\scriptsize$\pm0.00$} ($17.58${\scriptsize$\pm0.08$}) & $61.34${\scriptsize$\pm4.62$} ($84.84${\scriptsize$\pm1.61$}) & $50.40${\scriptsize$\pm2.40$} ($80.00${\scriptsize$\pm0.44$}) \\
SelfSup (Top)       & x & $83.64${\scriptsize$\pm0.16$} ($98.22${\scriptsize$\pm0.14$}) & $60.12${\scriptsize$\pm4.10$} ($76.64${\scriptsize$\pm0.70$}) & $71.30${\scriptsize$\pm3.43$} ($85.76${\scriptsize$\pm0.90$}) & $42.13${\scriptsize$\pm4.67$} ($74.27${\scriptsize$\pm0.55$}) \\
SelfSup (Hist)      & x & $78.20${\scriptsize$\pm2.54$} ($97.50${\scriptsize$\pm0.07$}) & $24.51${\scriptsize$\pm3.26$} ($50.54${\scriptsize$\pm8.80$}) & $71.06${\scriptsize$\pm2.12$} ($85.94${\scriptsize$\pm0.46$}) & $61.47${\scriptsize$\pm4.84$} ($80.63${\scriptsize$\pm0.91$}) \\
D2 (EL2N+ResNet) & x &
$84.79${\scriptsize$\pm0.63$} ($98.35${\scriptsize$\pm0.05$}) &
$40.22${\scriptsize$\pm4.86$} ($67.71${\scriptsize$\pm2.44$}) &
$75.46${\scriptsize$\pm2.12$} ($88.25${\scriptsize$\pm0.20$}) &
$70.93${\scriptsize$\pm2.54$} ($83.50${\scriptsize$\pm0.89$}) \\
D2 (SelfSup+ResNet) & x &
$83.43${\scriptsize$\pm0.51$} ($97.61${\scriptsize$\pm0.20$}) &
$64.07${\scriptsize$\pm7.23$} ($79.94${\scriptsize$\pm0.94$}) &
$70.60${\scriptsize$\pm3.56$} ($85.36${\scriptsize$\pm0.96$}) &
$60.53${\scriptsize$\pm1.15$} ($80.23${\scriptsize$\pm0.15$}) \\
D2 (EL2N+CLIP) & x &
$76.17${\scriptsize$\pm0.71$} ($97.88${\scriptsize$\pm0.05$}) &
$33.30${\scriptsize$\pm5.13$} ($64.89${\scriptsize$\pm2.97$}) &
$70.60${\scriptsize$\pm2.00$} ($85.88${\scriptsize$\pm0.82$}) &
$63.73${\scriptsize$\pm0.61$} ($81.53${\scriptsize$\pm0.55$}) \\
D2 (SelfSup+CLIP) & x &
$75.03${\scriptsize$\pm1.92$} ($97.82${\scriptsize$\pm0.08$}) &
$36.48${\scriptsize$\pm4.46$} ($67.23${\scriptsize$\pm1.68$}) &
$71.76${\scriptsize$\pm0.80$} ($86.05${\scriptsize$\pm0.44$}) &
$61.47${\scriptsize$\pm0.92$} ($79.67${\scriptsize$\pm0.93$}) \\
Random             & x & $75.47${\scriptsize$\pm0.77$} ($97.38${\scriptsize$\pm0.11$}) & $28.71${\scriptsize$\pm17.27$} ($61.84${\scriptsize$\pm1.12$}) & $71.30${\scriptsize$\pm1.75$} ($85.94${\scriptsize$\pm0.63$}) & $61.20${\scriptsize$\pm2.12$} ($80.57${\scriptsize$\pm0.46$}) \\
RGbal              & \checkmark & $\underline{86.55}${\scriptsize$\pm2.52$} ($97.97${\scriptsize$\pm0.01$}) & $\underline{66.02}${\scriptsize$\pm5.78$} ($81.58${\scriptsize$\pm0.28$}) & $\underline{78.70}${\scriptsize$\pm1.06$} ($89.35${\scriptsize$\pm0.82$}) & $\underline{76.27}${\scriptsize$\pm0.83$} ($86.57${\scriptsize$\pm0.25$}) \\
TCSL-CS           & x & $\boldsymbol{88.63}${\scriptsize$\pm0.61$} ($92.55${\scriptsize$\pm0.11$}) & $\boldsymbol{75.75}${\scriptsize$\pm8.18$} ($88.38${\scriptsize$\pm0.67$}) & $\boldsymbol{82.64}${\scriptsize$\pm0.69$} ($87.27${\scriptsize$\pm0.72$}) & $\boldsymbol{81.87}${\scriptsize$\pm1.40$} ($84.17${\scriptsize$\pm0.25$}) \\
\bottomrule
\end{tabular}
}
\caption{We compare the WGA and AVG of different coreset selection methods at $r\!=\!0.4$ across datasets. Results are averaged over $3$ seeds. We use CB ERM for retraining on all of the identified coresets. The best WGA for each dataset is shown in \textbf{bold} and the second best value is \underline{underlined}.
}
\label{tab:tab_0.4}
\end{table*}

\begin{table*}[!hbtp]
\centering
\resizebox{\textwidth}{!}{
\begin{tabular}{lccccc}
\toprule
\textbf{Method} & {\textbf{Group Info}} & \textbf{Waterbirds} & \textbf{cMNIST} & \textbf{MetaShift} & \textbf{UrbanCars-B} \\
& Train & WGA (AVG) & WGA (AVG) & WGA (AVG) & WGA (AVG) \\
\midrule
EL2N (Bot)     & x & $30.92${\scriptsize$\pm0.62$} ($93.17${\scriptsize$\pm0.39$}) & $0.00${\scriptsize$\pm0.00$} ($16.50${\scriptsize$\pm1.12$}) & $60.65${\scriptsize$\pm5.36$} ($80.21${\scriptsize$\pm2.23$}) & $32.93${\scriptsize$\pm8.34$} ($67.37${\scriptsize$\pm4.10$}) \\
EL2N (Top)     & x & $76.40${\scriptsize$\pm8.41$} ($98.25${\scriptsize$\pm0.11$}) & $51.31${\scriptsize$\pm3.94$} ($77.48${\scriptsize$\pm0.50$}) & $74.07${\scriptsize$\pm0.65$} ($87.79${\scriptsize$\pm0.22$}) & $68.67${\scriptsize$\pm2.44$} ($83.80${\scriptsize$\pm0.20$}) \\
EL2N (Hist)    & x & $81.31${\scriptsize$\pm1.22$} ($97.81${\scriptsize$\pm0.04$}) & $36.42${\scriptsize$\pm0.56$} ($62.40${\scriptsize$\pm0.26$}) & $72.22${\scriptsize$\pm1.70$} ($86.11${\scriptsize$\pm0.57$}) & $65.73${\scriptsize$\pm3.26$} ($82.57${\scriptsize$\pm0.45$}) \\
SelfSup (Bot)       & x & $49.25${\scriptsize$\pm1.49$} ($93.99${\scriptsize$\pm0.31$}) & $0.00${\scriptsize$\pm0.00$} ($15.36${\scriptsize$\pm1.20$}) & $67.59${\scriptsize$\pm0.65$} ($86.11${\scriptsize$\pm0.25$}) & $58.00${\scriptsize$\pm0.69$} ($82.50${\scriptsize$\pm0.61$}) \\
SelfSup (Top)       & x & $80.74${\scriptsize$\pm0.80$} ($98.16${\scriptsize$\pm0.03$}) & $\underline{55.62}${\scriptsize$\pm19.36$} ($78.88${\scriptsize$\pm2.55$}) & $\underline{75.46}${\scriptsize$\pm1.18$} ($86.75${\scriptsize$\pm0.43$}) & $53.60${\scriptsize$\pm8.65$} ($79.43${\scriptsize$\pm1.76$}) \\
SelfSup (Hist)      & x & $81.09${\scriptsize$\pm0.17$} ($97.77${\scriptsize$\pm0.02$}) & $39.27${\scriptsize$\pm8.77$} ($62.18${\scriptsize$\pm0.11$}) & $71.99${\scriptsize$\pm1.73$} ($86.52${\scriptsize$\pm0.46$}) & $66.00${\scriptsize$\pm4.92$} ($82.00${\scriptsize$\pm0.66$}) \\
D2 (EL2N+ResNet) & x &
$81.00${\scriptsize$\pm1.18$} ($98.27${\scriptsize$\pm0.07$}) &
$41.21${\scriptsize$\pm4.88$} ($68.75${\scriptsize$\pm7.38$}) &
$73.61${\scriptsize$\pm0.00$} ($87.73${\scriptsize$\pm0.27$}) &
$66.53${\scriptsize$\pm3.00$} ($82.93${\scriptsize$\pm1.68$}) \\
D2 (SelfSup+ResNet) & x &
$82.50${\scriptsize$\pm1.26$} ($97.99${\scriptsize$\pm0.06$}) &
$58.32${\scriptsize$\pm1.96$} ($77.02${\scriptsize$\pm1.72$}) &
$\underline{75.46}${\scriptsize$\pm1.45$} ($86.63${\scriptsize$\pm0.52$}) &
$61.20${\scriptsize$\pm3.67$} ($80.73${\scriptsize$\pm0.45$}) \\
D2 (EL2N+CLIP) & x &
$76.32${\scriptsize$\pm2.04$} ($97.66${\scriptsize$\pm0.04$}) &
$30.45${\scriptsize$\pm7.00$} ($65.55${\scriptsize$\pm4.61$}) &
$72.22${\scriptsize$\pm0.69$} ($85.76${\scriptsize$\pm0.17$}) &
$66.27${\scriptsize$\pm0.23$} ($81.67${\scriptsize$\pm0.21$}) \\
D2 (SelfSup+CLIP) & x &
$74.04${\scriptsize$\pm0.94$} ($97.70${\scriptsize$\pm0.03$}) &
$37.58${\scriptsize$\pm5.15$} ($66.85${\scriptsize$\pm3.38$}) &
$72.69${\scriptsize$\pm0.80$} ($86.28${\scriptsize$\pm0.17$}) &
$59.47${\scriptsize$\pm0.61$} ($79.03${\scriptsize$\pm1.29$}) \\
Random             & x & $78.04${\scriptsize$\pm1.10$} ($97.69${\scriptsize$\pm0.05$}) & $48.23${\scriptsize$\pm7.14$} ($68.54${\scriptsize$\pm2.20$}) & $71.53${\scriptsize$\pm0.57$} ($86.05${\scriptsize$\pm0.46$}) & $64.40${\scriptsize$\pm0.69$} ($82.00${\scriptsize$\pm0.44$}) \\
RGbal              & \checkmark & $\boldsymbol{83.54}${\scriptsize$\pm1.48$} ($97.91${\scriptsize$\pm0.11$}) & $54.13${\scriptsize$\pm1.87$} ($77.46${\scriptsize$\pm1.03$}) & $75.23${\scriptsize$\pm0.87$} ($88.31${\scriptsize$\pm0.22$}) & $\underline{72.80}${\scriptsize$\pm1.06$} ($85.03${\scriptsize$\pm0.84$}) \\
TCSL-CS           & x & $\underline{83.52}${\scriptsize$\pm0.33$} ($95.08${\scriptsize$\pm0.03$}) & $\boldsymbol{70.80}${\scriptsize$\pm1.61$} ($86.14${\scriptsize$\pm1.16$}) & $\boldsymbol{79.40}${\scriptsize$\pm1.64$} ($89.29${\scriptsize$\pm0.30$}) & $\boldsymbol{80.40}${\scriptsize$\pm1.06$} ($88.03${\scriptsize$\pm0.84$}) \\
\bottomrule
\end{tabular}
}
\caption{We compare the WGA and AVG of different coreset selection methods at $r\!=\!0.6$ across datasets. Results are averaged over $3$ seeds. We use CB ERM for retraining on all of the identified coresets. The best WGA for each dataset is shown in \textbf{bold} and the second best value is \underline{underlined}.
}
\label{tab:tab_0.6}
\end{table*}

\begin{table*}[!hbtp]
\centering
\resizebox{\textwidth}{!}{
\begin{tabular}{lccccc}
\toprule
\textbf{Method} & {\textbf{Group Info}} & \textbf{Waterbirds} & \textbf{cMNIST} & \textbf{MetaShift} & \textbf{UrbanCars-B} \\
& Train & WGA (AVG) & WGA (AVG) & WGA (AVG) & WGA (AVG) \\
\midrule
EL2N (Bot)     & x & $28.56${\scriptsize$\pm2.29$} ($94.81${\scriptsize$\pm0.13$}) & $0.00${\scriptsize$\pm0.00$} ($18.00${\scriptsize$\pm0.79$}) & $65.97${\scriptsize$\pm3.45$} ($82.06${\scriptsize$\pm1.66$}) & $32.67${\scriptsize$\pm7.49$} ($68.83${\scriptsize$\pm4.76$}) \\
EL2N (Top)     & x & $81.07${\scriptsize$\pm2.09$} ($98.07${\scriptsize$\pm0.03$}) & $\underline{55.48}${\scriptsize$\pm0.10$} ($74.35${\scriptsize$\pm0.76$}) & $72.92${\scriptsize$\pm0.57$} ($87.27${\scriptsize$\pm0.30$}) & $68.80${\scriptsize$\pm0.40$} ($83.70${\scriptsize$\pm0.36$}) \\
EL2N (Hist)    & x & $80.49${\scriptsize$\pm1.97$} ($97.86${\scriptsize$\pm0.02$}) & $40.11${\scriptsize$\pm16.23$} ($68.80${\scriptsize$\pm2.47$}) & $73.15${\scriptsize$\pm1.31$} ($86.86${\scriptsize$\pm0.22$}) & $66.93${\scriptsize$\pm1.85$} ($82.60${\scriptsize$\pm0.75$}) \\
SelfSup (Bot)  & x & $61.63${\scriptsize$\pm2.65$} ($96.13${\scriptsize$\pm0.18$}) & $0.00${\scriptsize$\pm0.00$} ($16.22${\scriptsize$\pm1.95$}) & $70.37${\scriptsize$\pm0.33$} ($86.69${\scriptsize$\pm0.30$}) & $64.13${\scriptsize$\pm0.83$} ($83.70${\scriptsize$\pm0.17$}) \\
SelfSup (Top)  & x & $80.89${\scriptsize$\pm1.06$} ($98.09${\scriptsize$\pm0.01$}) & $55.35${\scriptsize$\pm4.53$} ($74.87${\scriptsize$\pm0.74$}) & $74.31${\scriptsize$\pm0.57$} ($87.27${\scriptsize$\pm0.30$}) & $64.00${\scriptsize$\pm1.60$} ($81.57${\scriptsize$\pm0.59$}) \\
SelfSup (Hist) & x & $\underline{82.24}${\scriptsize$\pm1.10$} ($97.93${\scriptsize$\pm0.17$}) & $49.94${\scriptsize$\pm10.04$} ($71.10${\scriptsize$\pm1.54$}) & $71.99${\scriptsize$\pm1.31$} ($86.75${\scriptsize$\pm0.78$}) & $67.47${\scriptsize$\pm1.80$} ($82.83${\scriptsize$\pm0.40$}) \\
D2 (EL2N+ResNet) & x &
$81.26${\scriptsize$\pm1.94$} ($98.09${\scriptsize$\pm0.04$}) &
$45.95${\scriptsize$\pm2.50$} ($68.54${\scriptsize$\pm0.41$}) &
$72.92${\scriptsize$\pm0.69$} ($87.21${\scriptsize$\pm0.27$}) &
$67.60${\scriptsize$\pm3.82$} ($82.77${\scriptsize$\pm1.11$}) \\
D2 (SelfSup+ResNet) & x &
$79.75${\scriptsize$\pm2.29$} ($97.93${\scriptsize$\pm0.03$}) &
$47.69${\scriptsize$\pm9.67$} ($75.62${\scriptsize$\pm2.21$}) &
$\boldsymbol{74.77}${\scriptsize$\pm1.06$} ($87.38${\scriptsize$\pm0.36$}) &
$57.60${\scriptsize$\pm2.08$} ($80.03${\scriptsize$\pm0.21$}) \\
D2 (EL2N+CLIP) & x &
$77.05${\scriptsize$\pm1.65$} ($97.75${\scriptsize$\pm0.03$}) &
$37.67${\scriptsize$\pm7.74$} ($63.48${\scriptsize$\pm4.71$}) &
$68.98${\scriptsize$\pm2.12$} ($85.07${\scriptsize$\pm0.80$}) &
$64.80${\scriptsize$\pm1.06$} ($80.90${\scriptsize$\pm0.61$}) \\
D2 (SelfSup+CLIP) & x &
$73.62${\scriptsize$\pm0.18$} ($97.76${\scriptsize$\pm0.12$}) &
$34.48${\scriptsize$\pm6.20$} ($64.26${\scriptsize$\pm4.38$}) &
$70.60${\scriptsize$\pm1.06$} ($85.36${\scriptsize$\pm0.44$}) &
$59.33${\scriptsize$\pm2.72$} ($79.43${\scriptsize$\pm1.02$}) \\
Random         & x & $79.44${\scriptsize$\pm0.66$} ($97.83${\scriptsize$\pm0.02$}) & $36.02${\scriptsize$\pm1.45$} ($69.29${\scriptsize$\pm4.13$}) & $72.45${\scriptsize$\pm0.65$} ($86.92${\scriptsize$\pm0.22$}) & $61.87${\scriptsize$\pm3.00$} ($82.10${\scriptsize$\pm0.82$}) \\
RGbal          & \checkmark & $81.52${\scriptsize$\pm1.06$} ($97.88${\scriptsize$\pm0.11$}) & $48.29${\scriptsize$\pm16.12$} ($76.51${\scriptsize$\pm0.12$}) & $72.45${\scriptsize$\pm1.43$} ($87.62${\scriptsize$\pm0.08$}) & $\underline{70.27}${\scriptsize$\pm0.46$} ($84.20${\scriptsize$\pm0.70$}) \\
TCSL-CS        & x & $\boldsymbol{83.71}${\scriptsize$\pm1.34$} ($96.87${\scriptsize$\pm0.03$}) & $\boldsymbol{62.09}${\scriptsize$\pm1.71$} ($82.20${\scriptsize$\pm2.72$}) & $\underline{74.54}${\scriptsize$\pm0.33$} ($89.00${\scriptsize$\pm0.16$}) & $\boldsymbol{73.87}${\scriptsize$\pm4.69$} ($86.77${\scriptsize$\pm0.90$}) \\
\bottomrule
\end{tabular}
}
\caption{We compare the WGA and AVG of different coreset selection methods at $r\!=\!0.8$ across datasets. Results are averaged over $3$ seeds. We use CB ERM for retraining on all of the identified coresets. The best WGA for each dataset is shown in \textbf{bold} and the second best value is \underline{underlined}.
}
\label{tab:tab_0.8}
\end{table*}

%% file: appendix/sec_h.tex
\section{Additional Ablation Studies}\label{sec:app_abb_study}

In Figure \ref{fig:comp_densities}, we compare the distribution of EL2N scores with our $\text{TCSL}_s$ and $\text{TCSL}_c$ scores on Waterbirds. While both EL2N and $\text{TCSL}_s$ consistently assign higher scores to minority-group samples (i.e., without spurious background correlations), $\text{TCSL}_c$ remains largely invariant to background and assigns comparable scores across groups. Next, we provide additional ablation studies for TCSL-CS. As the TCSL score consists of two components, $\text{TCSL}_c$ and $\text{TCSL}_s$, we evaluate the need for each part separately. Finally, we present ablation studies on the choice of $B$ and $T_s$ in our algorithm, to illustrate the robustness of our proposed coreset selection method. 

\begin{figure*}[t]
    \centering
    \includegraphics[width=\linewidth]{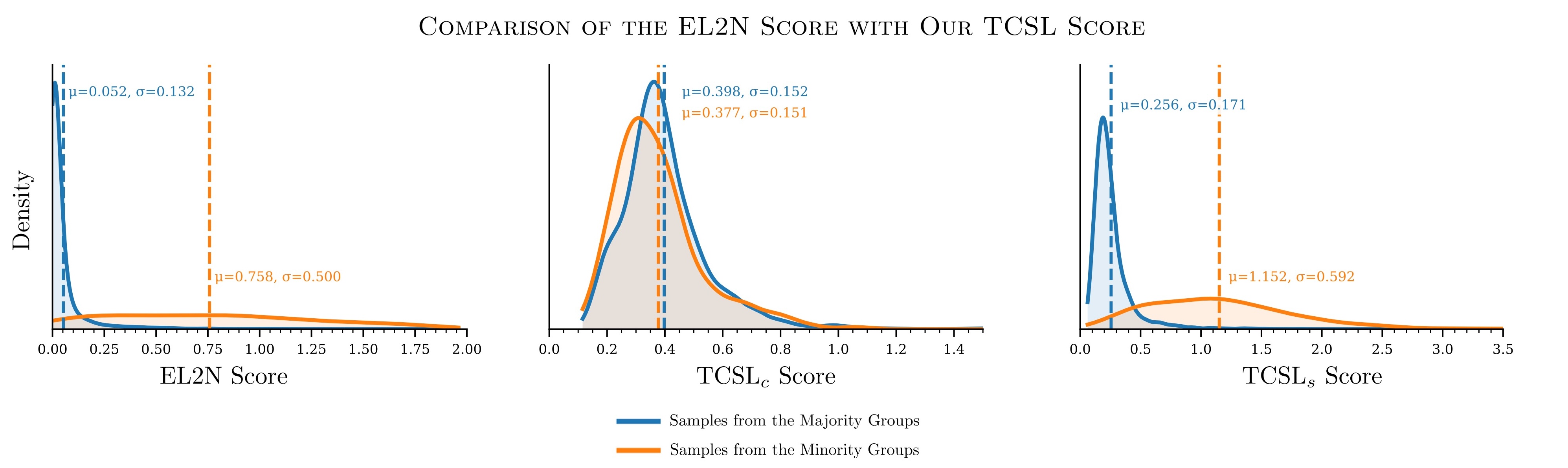}
    \caption{Density Comparison of EL2N and TCSL Scores.}
    \label{fig:comp_densities}
\end{figure*}

First, we assume that the group identification stage based on the $\text{TCSL}_s$ score in \textbf{Algorithm~\ref{alg:tcsl_coreset}} of the main paper is replaced with the exact group labels. Hence, our coreset selection algorithm first selects samples from the minority groups and then fills the remaining quota based on the $\text{TCSL}_c$ scores of the majority group. In this way, we are able to demonstrate the performance gains achieved by selecting based on the $\text{TCSL}_c$ score rather than using random selection as in RGbal. We illustrate the results in Table \ref{tab:tab_tcslc}. Since $\alpha > 0.9$ for all datasets except MetaShift, we set $r\!=\!0.1$ as in the main paper. For MetaShift, as $\alpha = 0.882$, we set $r\!=\!0.2$. The results show that TCSL-CS improves performance on all datasets when the group labels are made available. We note that the performance of TCSL-CS could be further improved with group labels, since our $\text{TCSL}_c$ score is still computed under a debiased training scheme constructed without access to group labels. With access to group labels, one could build even stronger biased and debiased model training schemes, which would in turn improve the computation of the $\text{TCSL}_c$ score.

Next, we compare the $\text{TCSL}_s$ scores obtained from our proposed biased model structure given in \textbf{Algorithm 1} of the main paper, with the scores obtained from a standard baseline biased model from the literature. Specifically, we train a biased model using the generalized cross entropy (GCE) loss \citep{GCE} for the same number of epochs as our biased model and compute the $\text{TCSL}_s$ scores for both cases. 

As our goal is to capture information related only to the spurious feature part of the image, we compute the cosine similarity between the computed $\text{TCSL}_s$ scores and the CSL scores obtained by manually removing and training on only the spurious or core feature parts of the images. We choose Waterbirds for this study, since the dataset is artificially constructed, allowing for clean separation of the feature parts. The results are presented in Table \ref{tab:tab_tcsls}.
Since GCE requires tuning the hyperparameter $q$, we report results for $q \in \{0.7, 0.8, 0.9\}$. As shown, the scores obtained with our biased model achieve higher similarity to $\text{CSL}(x^s)$ and lower similarity to $\text{CSL}(x^c)$. This demonstrates that our biased model architecture is a more suitable choice for our setting than the widely used GCE-based biased training.

To demonstrate that the performance gains obtained from training on our identified coresets are not specific to a particular architecture, we present cross-architectural results in Table~\ref{tab:tab_cross_arc_wga}. Models are trained on coresets selected by TCSL-CS ($r=0.1$) using a ResNet backbone (ResNet18 for cMNIST and ResNet50 for all other datasets). We report WGA improvements relative to ERM baselines of the same target architectures. The results indicate that the selected coresets are architecture-agnostic, yielding consistent performance gains across different model architectures.

\begin{table*}[!hbtp]
\centering
\resizebox{\linewidth}{!}{
\begin{tabular}{lccccc}
\toprule
\textbf{Method} & {\textbf{Group Info}} & \textbf{Waterbirds} & \textbf{cMNIST} & \textbf{MetaShift} & \textbf{UrbanCars-B} \\
& Train & WGA (AVG) & WGA (AVG) & WGA (AVG) & WGA (AVG) \\
\midrule
TCSL-CS           & x & $91.91${\scriptsize$\pm0.35$} ($92.83${\scriptsize$\pm0.55$}) & $\underline{83.37}${\scriptsize$\pm1.33$} ($91.76${\scriptsize$\pm0.44$}) & $76.85${\scriptsize$\pm4.24$} ($88.60${\scriptsize$\pm1.12$}) & $\underline{84.27}${\scriptsize$\pm1.67$} ($86.93${\scriptsize$\pm0.78$}) \\
TCSL-CS           & \checkmark & $\underline{92.06}${\scriptsize$\pm0.12$} ($93.32${\scriptsize$\pm0.43$}) & $\boldsymbol{88.26}${\scriptsize$\pm2.07$} ($94.01${\scriptsize$\pm1.67$}) & $\boldsymbol{85.42}${\scriptsize$\pm4.42$} ($88.02${\scriptsize$\pm1.84$}) & $\boldsymbol{85.20}${\scriptsize$\pm0.14$} ($86.70${\scriptsize$\pm0.07$}) \\
RGbal              & \checkmark & $\boldsymbol{92.67}${\scriptsize$\pm0.58$} ($93.68${\scriptsize$\pm0.41$}) & $71.68${\scriptsize$\pm13.83$} ($87.59${\scriptsize$\pm2.24$}) & $\underline{80.79}${\scriptsize$\pm1.06$} ($87.85${\scriptsize$\pm1.22$}) & $84.00${\scriptsize$\pm1.20$} ($86.57${\scriptsize$\pm0.06$}) \\
\bottomrule
\end{tabular}
}
\caption{We conduct an ablation study on the $\text{TCSL}_c$ score. Since $\alpha > 0.9$ for all datasets except MetaShift, we set $r\!=\!0.1$ as in the main paper. For MetaShift, as $\alpha = 0.882$, we set $r\!=\!0.2$. We use CB ERM for retraining on all of the identified coresets. The best WGA for each dataset is shown in \textbf{bold} and the second best value is \underline{underlined}.}
\label{tab:tab_tcslc}
\end{table*}

\begin{table}[!hbtp]
\centering
\small % or \footnotesize if still big
\setlength{\tabcolsep}{4pt} % tighten columns slightly
\begin{tabular}{lcccc}
\toprule
\textbf{Method} & \textbf{Ours} & \textbf{GCE(q=0.7)} & \textbf{GCE(q=0.8)} & \textbf{GCE(q=0.9)} \\
\midrule
$\text{CSL}(x^c)$ & 0.568 & 0.637 & 0.639 & 0.641 \\
$\text{CSL}(x^s)$ & 0.823 & 0.769 & 0.763 & 0.756 \\
\bottomrule
\end{tabular}
\caption{Cosine similarity between the CSL scores obtained with models trained on only the core and spurious feature parts of images and the $\text{TCSL}_s$ scores obtained from biased models trained with different methods. We compare our proposed biased model ($f_s$) against biased models trained with the generalized cross-entropy (GCE) loss.}
\label{tab:tab_tcsls}
\end{table}

\begin{table}[!hbtp]
\centering
\small
\setlength{\tabcolsep}{4pt} % slightly tighter spacing
\begin{tabular}{lcccc}
\toprule
\textbf{Dataset} & \textbf{Original} & \textbf{ResNet101} & \textbf{InceptionV3} & \textbf{DenseNet121} \\
\midrule
Waterbirds   & +10.76 & +9.97 & +8.10 & +15.81 \\
cMNIST       & +26.76 & +24.09 & +28.20 & +25.14 \\
MetaShift    & +6.25  & +4.86 & +4.86 & +4.08 \\
UrbanCars-B  & +18.27 & +18.80 & +15.60 & +13.30 \\
\bottomrule
\end{tabular}
\caption{We evaluate performance gains relative to ERM baselines across different architectures. Each model is trained on the coresets $(r=0.1)$ identified by TCSL-CS, where the coreset selection is conducted using a ResNet backbone (ResNet18 for cMNIST and ResNet50 for all other datasets). The results show that the selected coresets are architecture-agnostic and transfer effectively across different model architectures.}
\label{tab:tab_cross_arc_wga}
\end{table}

In Tables \ref{tab:b_ablation} and \ref{tab:ts_ablation} we vary $B$ and $T_s$ in our algorithm, respectively, to support the effectiveness of our proposed coreset selection method. The results show that performance remains robust across these choices and consistently outperforms the strongest baselines on all four datasets.

\begin{table*}[!hbtp]
\centering
\resizebox{0.8\linewidth}{!}{
\begin{tabular}{lccc}
\toprule
\textbf{Method} & \textbf{Waterbirds} & \textbf{MetaShift} & \textbf{UrbanCars-B} \\
& WGA (AVG) & WGA (AVG) & WGA (AVG) \\
\midrule
$B=25$  & 91.22$\pm$0.31 (92.75$\pm$0.76) & \textbf{79.49$\pm$2.46} (87.85$\pm$0.49) & \underline{85.00$\pm$1.13} (87.60$\pm$0.14) \\
$B=50$  & \textbf{91.91$\pm$0.35} (92.83$\pm$0.55) & 79.40$\pm$2.23 (84.95$\pm$1.52) & 84.27$\pm$1.67 (86.93$\pm$0.78) \\
$B=100$ & \underline{91.75$\pm$0.87} (92.53$\pm$1.45) & \underline{79.48$\pm$3.44} (85.59$\pm$1.23) & \textbf{85.60$\pm$0.28} (87.35$\pm$0.07) \\
\bottomrule
\end{tabular}
}
\caption{We compare the WGA and AVG for different $B$ at $r=0.1$ across datasets. Results are averaged over 3 seeds. The best WGA for each dataset is shown in \textbf{bold} and the second-best value is \underline{underlined}. $B=50$ represents what was originally reported.}
\label{tab:b_ablation}
\end{table*}

\begin{table*}[!hbtp]
\centering
\resizebox{\linewidth}{!}{
\begin{tabular}{lcccc}
\toprule
\textbf{Method} & \textbf{Waterbirds} & \textbf{cMNIST} & \textbf{MetaShift} & \textbf{UrbanCars-B} \\
& WGA (AVG) & WGA (AVG) & WGA (AVG) & WGA (AVG) \\
\midrule
$T_s=T/12$ & \textbf{92.28$\pm$0.50} (93.18$\pm$0.49) & 82.88$\pm$2.24 (88.78$\pm$0.06) & \textbf{80.86$\pm$1.44} (85.97$\pm$1.97) & \underline{83.72$\pm$1.62} (87.17$\pm$0.61) \\
$T_s=T/10$ & 91.91$\pm$0.35 (92.83$\pm$0.55) & \underline{83.37$\pm$1.33} (91.76$\pm$0.44) & 79.40$\pm$2.23 (84.95$\pm$1.52) & \textbf{84.27$\pm$1.67} (86.93$\pm$0.78) \\
$T_s=T/8$  & \underline{92.06$\pm$0.52} (92.78$\pm$0.36) & \textbf{84.07$\pm$1.94} (92.75$\pm$0.77) & \underline{80.14$\pm$2.89} (85.60$\pm$1.86) & 81.85$\pm$1.80 (85.60$\pm$1.04) \\
\bottomrule
\end{tabular}
}
\caption{We compare the WGA and AVG for different $T_s$ at $r=0.1$ across datasets. Results are averaged over 3 seeds. The best WGA for each dataset is shown in \textbf{bold} and the second-best value is \underline{underlined}. $T_s=T/10$ represents what was originally reported.}
\label{tab:ts_ablation}
\end{table*}

%% file: appendix/sec_rcm_anal.tex
\section{Computational Efficiency}\label{sec:app_rcsm}

Below, we analyze the runtime, computational complexity, and memory usage of our algorithm.

\subsection{Complexity Analysis}

We assume that the cost of a single forward pass through a given deep neural network is $F$, and that the cost of one backpropagation step is $B$. Thus, the cost of one forward-backward pass is $F+B$. Let $T_s$ denote the number of epochs used to train the biased model, $T_c$ the number of epochs used to train the debiased model, and $n$ the size of the training dataset.

Under standard empirical risk minimization (ERM), training for $T$ epochs incurs a computational cost of $Tn(F+B)$. In TCSL, we use $T_s=T/10$ and $T_c=T$ by default. The TCSL scores are obtained during training without additional forward passes, since per-sample losses are already computed as part of optimization. Hence, the total computational cost of TCSL is $1.1Tn(F+B)$. Assuming $B \approx 2F$, ERM requires $\approx 3TnF$ operations, whereas TCSL requires $\approx 3.3TnF$ operations.

Common two-stage algorithms, including LC, CNC, LfF, JTT, and ULA, typically train two models for $T$ epochs each. Their total computational cost is therefore $2Tn(F+B)=6TnF$. By contrast, TCSL trains the biased model for only a fraction of the full training budget, resulting in lower computational overhead.

We note that the debiased model $f_c$ trained within the TCSL framework can itself be used as a final model, since it follows the structure of two-stage and logit-correction based approaches. However, the primary objective of TCSL is to identify a \textbf{debiased coreset} such that a standard ERM model trained on this subset achieves state-of-the-art performance. Consequently, evaluating the quality of the selected coreset requires an additional ERM training run.

\subsection{Runtime Analysis}

We report the average runtime per training epoch for TCSL across all four datasets, along with the total runtime of the coreset selection procedure in Table~\ref{tab:tab_runtime}.

\begin{table}[t]
\centering
\begin{tabular}{lcc}
\toprule
\textbf{Dataset} & \textbf{Training (1 Epoch)} & \textbf{Coreset Selection (Total)} \\
\midrule
Waterbirds  & 6.576s & 1.161s \\
cMNIST      & 25.30s & 1.735s \\
MetaShift   & 3.201s & 0.732s \\
UrbanCars-B & 9.610s & 1.392s \\
\bottomrule
\end{tabular}
\caption{Runtime of each stage of the TCSL algorithm. All measurements were obtained on a single NVIDIA A100 GPU.}\label{tab:tab_runtime}
\end{table}

\subsection{Memory Usage Analysis}

Our implementation of wKMeans takes the biased model's loss trajectories as input, which requires storing $T_s \cdot n$ scalar values that are then used to compute the TCSL$_s$ scores. The TCSL$_c$ scores are computed as the average per-sample losses over the debiased model's training trajectory. These values are accumulated online during training, so only $n$ scalar values are stored. Hence, the total memory overhead is $(T_s + 1) \cdot n$ scalar values. Compared with the datasets used in our experiments, which typically require $224 \cdot 224 \cdot 3 \cdot n$ scalar values for image storage, this introduces a negligible memory overhead of approximately $0.02$\% for $T=300$. Furthermore, the losses can be stored in CPU memory to reduce GPU memory pressure.

%% file: appendix/sec_i.tex
% \section{Limitations and Future Work}\label{sec:app_limit_future}

% While our work is the first coreset selection algorithm specifically designed to mitigate spurious correlations and effectively select coresets to represent each group within the data, we assume only two groups: a majority group, where the spurious attribute agrees with the class label and a minority group, where the spurious attribute disagrees with the class label. Although we demonstrate that our method effectively selects coresets while overcoming spurious correlations across various datasets and settings, we do not consider the more complex scenario in which multiple spurious correlations, and hence multiple minority groups, are present in the data. We plan to extend our work to more complex settings, such as the presence of multiple spurious correlations and case where the same spurious correlation can appear in multiple classes.